\documentclass{article}
\usepackage{ijcai23}

\usepackage{times}
\usepackage{soul}
\usepackage{url}
\usepackage[hidelinks]{hyperref}
\usepackage[utf8]{inputenc}
\usepackage[small]{caption}
\usepackage{graphicx}
\usepackage{amsmath}
\usepackage{amsthm}
\usepackage{booktabs}
\usepackage{algorithm}
\usepackage{algorithmic}
\usepackage[switch]{lineno}

\usepackage{caption}
\usepackage{subcaption}
\usepackage{mathtools}
\usepackage{amsfonts}

\newcommand{\QP}[1]{\textcolor{blue}{{\bf QP:} #1}}

\usepackage[table]{xcolor}
\definecolor{silver}{rgb}{0.75, 0.75, 0.75}
\definecolor{amethyst}{rgb}{0.54, 0.17, 0.89}
\definecolor{palecornflowerblue}{rgb}{0.67, 0.8, 0.94}
\definecolor{bubbles}{rgb}{0.91, 1.0, 1.0}
\definecolor{champagne}{rgb}{0.97, 0.91, 0.81}
\newcolumntype{a}{>{\columncolor{champagne}}c}

\newtheorem{assumption}{Assumption}
\newtheorem{lemma}{Lemma}

\usepackage[usestackEOL]{stackengine}
\usepackage{enumitem}

\DeclareMathOperator*{\argmax}{arg\,max}            
\usepackage{cellspace}
\newtheorem{theorem}{Theorem}

\title{RAIN: RegulArization on Input and Network for Black-Box Domain Adaptation }
\author{
Qucheng Peng$^1$
\and
Zhengming Ding$^2$\and
Lingjuan Lyu$^3$\and
Lichao Sun$^4$\And
Chen Chen$^1$
\affiliations
$^1$Center for Research in Computer Vision, University of Central Florida\\
$^2$Department of Computer Science, Tulane University\\$^3$Sony AI\\$^4$Lehigh University
\emails
qucheng.peng@knights.ucf.edu, chen.chen@crcv.ucf.edu
}
\begin{document}

\maketitle

\begin{abstract}

Source-Free domain adaptation transits the source-trained model towards target domain without exposing the source data, trying to dispel these concerns about data privacy and security. However, this paradigm is still at risk of data leakage due to adversarial attacks on the source model. Hence, the Black-Box setting only allows to use the outputs of source model, but still suffers from overfitting on the source domain more severely due to source model's unseen weights. In this paper, we propose a novel approach named RAIN (\textbf{R}egul\textbf{A}rization on \textbf{I}nput and \textbf{N}etwork) for Black-Box domain adaptation from both input-level and network-level regularization. For the input-level, we design a new data augmentation technique as Phase MixUp, which highlights task-relevant objects in the interpolations, thus enhancing input-level regularization and class consistency for target models. For network-level, we develop a Subnetwork Distillation mechanism to transfer knowledge from the target subnetwork to the full target network via knowledge distillation, which thus alleviates overfitting on the source domain by learning diverse target representations. Extensive experiments show that our method achieves state-of-the-art performance on several cross-domain benchmarks under both single- and multi-source black-box domain adaptation. %Codes are available at \url{https://github.com/davidpengucf/RAIN}.

%Source-Free domain adaptation aims to adapt a source-trained model to a target domain without exposing the source data, addressing concerns about data privacy and security. However, this approach is still vulnerable to adversarial attacks on the source model. In the Black-Box setting, only the outputs of the source model can be used, but this leads to increased overfitting on the source domain. In this paper, we propose a new approach, RAIN (\textbf{R}egul\textbf{A}tion on \textbf{I}nput and \textbf{N}etwork), for Black-Box domain adaptation that uses both input-level and network-level regularization. Our input-level regularization technique, Phase MixUp, enhances input-level regularization and class consistency for target models by highlighting task-relevant objects in interpolations. Our network-level regularization mechanism, Subnetwork Distillation, transfers knowledge from the target subnetwork to the full target network, reducing overfitting on the source domain by learning diverse target representations. Our method outperforms state-of-the-art techniques on several cross-domain benchmarks for both single- and multi-source black-box domain adaptation. \QP{Chat-GPT version}
\end{abstract}

\section{Introduction}

%Deep learning approaches have achieved remarkable success in various visual tasks \cite{noh2015learning,chen2017deeplab,redmon2016you,ren2015faster} when there are a large amount of labeled data and the training and test data lie in the same distribution. However, if the training and test data come from different domains, these methods cannot generalize well to the test domain due to the presence of \emph{domain shift} \cite{ben2010theory}. To alleviate this issue, domain adaptation \cite{wang2018deep} is proposed to transfer knowledge from the labeled training data that form the \emph{source} domain to the unlabeled test data that form the \emph{target} domain. Most existing domain adaptation methods \cite{long2015learning,tzeng2014deep,sun2016deep,ganin2015unsupervised,saito2018maximum} need to access labeled data from source, but this paradigm cannot be satisfied under certain conditions \cite{kairouz2021advances}. For example, personal medical reports from hospitals or personal information records from social media must not be shared due to regulations that protect individual privacy. Therefore, a new setting --Source-Free domain adaptation \cite{liang2020we}-- emerges, where source data is completely unavailable when training on target. 

\begin{figure}[htp]
  %\vspace{-0.53cm}
  
  \centering  \includegraphics[width=0.95\linewidth]{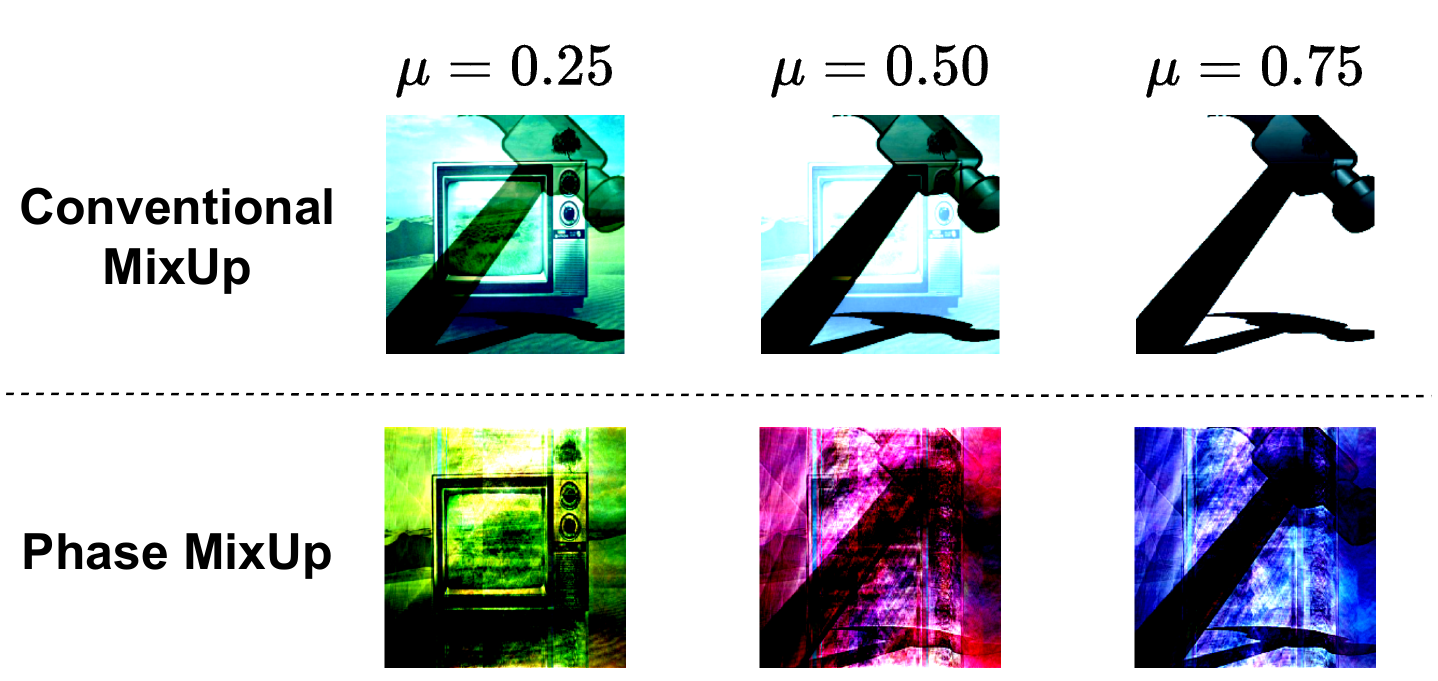}
  %\vspace{-8pt}
  \caption{Comparisons between the conventional MixUp \protect\cite{zhang2018mixup} and our Phase MixUp. (Best viewed in color.) $\mu$ is the ratio of the ``hammer'' image fused into the ``television'' image. Note that when $\mu = 0.50$, in Conventional MixUp the dark shadow outweighs ``television'', while in Phase MixUp both the core objects can be highlighted. Besides, when $\mu = 0.75$, the ``television" is completely unseen in Conventional MixUp, while it still exists in Phase MixUp. These demonstrate that Phase MixUp can alleviate the interference of background and style information.}
  %\label{fig:framework}
  \label{fig:compare}
  %\vspace{-6mm}
\end{figure}

Domain adaptation \cite{wang2018deep} is proposed to transfer knowledge from the labeled training data that form the \emph{source} domain to the unlabeled test data that form the \emph{target} domain. Owing to the concerns about data privacy and security, a new setting --Source-Free domain adaptation \cite{liang2020we}-- emerges, where source data are completely unavailable when adapting to the target. Even so, there are still potential risks if the source models are visible. Some works like dataset distillation \cite{wang2018dataset} and DeepInversion \cite{yin2020dreaming} may recover data from the model through adversarial attacks. In such a case, \textbf{Black-Box domain adaptation} \cite{zhang2021unsupervised} is proposed to consider the source models are completely unseen and only model outputs are accessible for target adaptation, which is a more strict version of Source-Free domain adaptation.

Due to the limited access to the source model, the only way we obtain source information is by doing knowledge distilling between the source model and the target model. In the original Source-Free setting, we can alleviate domain shifts by updating the source model gradually, keeping the transferrable parameters while replacing the untransferrable ones. However, in the Black-box scenario, only the source model's outputs are exposed, i.e., the source information cannot be disentangled as target-relevant and target-irrelevant parts as Source-Free does. The overfitting on the source domain in Black-Box is much stronger than that in Source-free, which greatly undermines the target models' performance. \cite{liang2021distill} observes this problem and uses conventional MixUp \cite{zhang2018mixup} as regularization for better generalization. However, there exist problems with regularization methods like MixUp or CutMix \cite{yun2019cutmix} because \textit{they are conducted on both input and label levels}. In the target adaptation process, we do not have accurate labels but the pseudo labels as substitutes and the linear behaviors learned from these noisy pseudo-labels due to domain shifts could have negative impacts on the generalization and degrade the model's performance (details in Table \ref{tab:ab-aug}).

Based on the discussions above, we develop a new and effective data augmentation method for domain adaptation named Phase MixUp, which regularizes ML model training only from the \emph{input aspect}. Apart from inhibiting potential noises from pseudo-labels, our Phase MixUp can highlight task-relevant objects and avoid negative impacts from background information as well. Conventional MixUp directly blends one image's pixels into another image \cite{wu2020dual,liang2021distill}. But this kind of combination is too simple to highlight the objects that we want to classify, especially when the image owns complex background and style information. For example, the top of Fig. \ref{fig:compare} shows the conventional MixUp between the ``television" image and the ``hammer" image. The ``hammer" image has a very strong dark shadow that outweighs another object ``television" when $\mu = 0.50$. Besides, when $\mu = 0.75$, the ``television" is completely unseen. %The ``television" image has a relatively complex background (e.g. prairie and sky, see Fig. \ref{fig:framework}) and style information. 
In such cases, background elements like dark shadow will have a negative impact on conventional MixUp as they distract the attention of core objects and the model tends to connect the ``hammer" with the shadow instead. Frequency decomposition proves to be a useful tool to disentangle object information and background information from an image \cite{yang2020fda,liu2021feddg}, \textit{since amplitude spectra contain most background information, while phase spectra are related to object information.} Therefore, we propose Phase MixUp (Fig. \ref{fig:framework}b) to capture the key objects and reduce background interference at the same time, as the bottom of Fig. \ref{fig:compare} shows. By mixing their phase spectra, we can focus more on the two core objects ``television" and ``hammer" and weaken background information like the shadow. Even under a more extreme case like $\mu = 0.75$, the ``television" still exists in the mixed image of Phase MixUp. The augmented image will attend further training to enhance the target model's class consistency.

What's more, despite the fact that input- and label-level regularizations have received enough attention in domain adaptation, regularization from the \emph{network aspect} is overlooked. Specifically, in the Black-Box setting, only the source model's outputs are accessible, which leads to more severe overfitting of target networks on source information. Because target networks have to learn from the outputs produced by source models without detailed calibrations on network weights as the Source-Free setting does. Therefore, a network-level regularization technique on target networks is necessary. To this end, we propose a novel method for domain adaptation called Subnetwork Distillation, which aims to regularize the full target network with the help of its subnetwork and calibrate the full target networks gradually. %This process is achieved via self-knowledge distillation. 
We slim the widths of the target network to get its subnetwork, which has a smaller scale than the full network, hence less likely overfitting to the source domain. By transferring knowledge from the target subnetwork to the full target network, the original full network captures diverse representations from the target domain with a better generalization ability.  %\QP{A little bit short, maybe need to add more.}In order to ensure that target subnetworks deliver the correct and convincing information, the divergences between the full network and subnetworks are maintained to be small. Besides, to ensure the self-knowledge distillation meaningful, it is necessary to guarantee that the full network and subnetwork contain varied knowledge, so we quantize the different knowledge learned from the original networks and subnetworks, enlarging their gradient discrepancy to ensure the subnetworks actually serve as a regularizer.

{Our contributions are summarized in three aspects:
\setlist{nolistsep}
\begin{itemize}[noitemsep,leftmargin=*] 
    \item We propose Phase MixUp as a new \emph{input-level} regularization scheme that helps the model enhance class consistency with more task-relevant object information, thus obtaining more robust representations for the target domain.
    \item We introduce a novel \emph{network-level} regularization technique called Subnetwork Distillation that assists the target model to learn diverse representations and transfers knowledge from the model's partial structures to avoid overfitting on Black-Box source models.
    \item We conduct extensive experimental results on several benchmark datasets with both Single-Source and Multi-Source settings, showing that our approach achieves state-of-the-art performance compared with the latest methods for Black-Box domain adaptation. 
\end{itemize}}

\section{Related Work}

\paragraph{Domain Adaptation.} %Motivated by the work from \cite{ben2010theory}, a flurry of domain adaptation methods have been proposed.
Metric-based and GAN-based approaches are the two major routes in single-source domain adaptation. The metric-based methods measure the discrepancy between the source and target domain explicitly. \cite{long2015learning} uses maximum mean discrepancy, while \cite{tzeng2014deep} applies deep domain confusion. Recent work like \cite{deng2021cluster} jointly makes clustering and discrimination for alignment together with contrastive learning. The GAN-based methods originate from \cite{goodfellow2014generative} and build a min-max game for two players related to the source and target domains. \cite{ganin2015unsupervised} adopts domain to confuse the two players, while \cite{saito2018maximum} uses classifier discrepancy as the objective, and \cite{tang2020unsupervised} reveals that discriminative clustering on target will benefit the adaptation. However, the general domain adaptation setting needs access to source data, which raises concerns about data privacy and security, so Source-Free adaptation is proposed.

%\vspace{-3.5pt}
%\vspace{0.1cm}
\paragraph{Source-Free Domain Adaptation.} Based on the work of \cite{li2020ma,liang2020we}, Source-Free has become the mainstream paradigm for alleviating concerns about data privacy and security in domain adaptation. There are two technique routes under the source-free setting: self-supervision and virtual source transfer. For the self-supervised methods, \cite{liang2020we} is the most representative one, which introduces information maximization to assist adaptation. \cite{xia2021adaptive} treats the problem from a discriminative perspective and adds a specific representation learning module to help the generalization, and \cite{chen2022contrastive} proposes the online pseudo label refinement. As for virtual source methods, most of them build GANs to generate virtual source data. \cite{kurmi2021domain} uses conditional GAN to generate new samples, while \cite{hou2021visualizing} provides interesting visualizations for unseen knowledge and \cite{li2020ma} applies collaborative GAN to achieve better generations. But Source-Free is still at risk of data leakage due to adversarial attacks to visible model weights, and that leads to a more strict setting -- Black-Box domain adaptation.

\begin{figure*}[htp]
  %\vspace{-0.5cm}
  \centering  \includegraphics[width=0.96\linewidth]{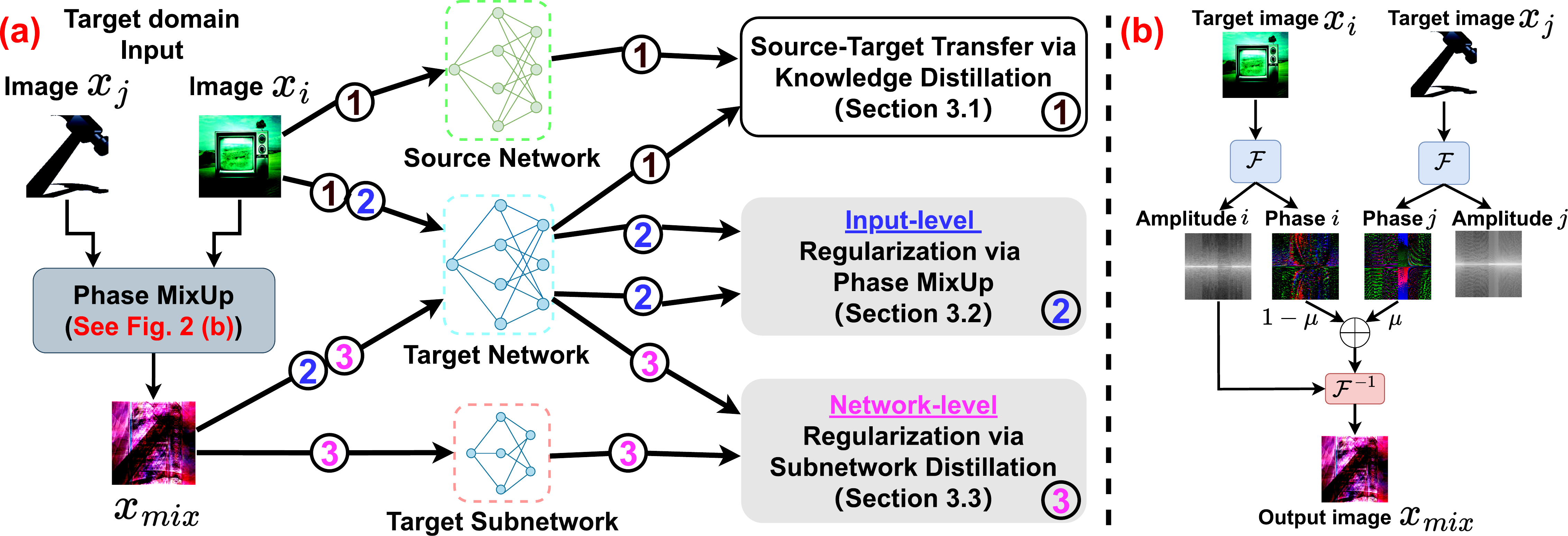}
  %\vspace{-7pt}
  \caption{(a) Overview of our proposed framework. (b) The process of Phase MixUp (Best viewed in color).  In (a), the arrows marked with ``1" are related to the knowledge distillation described in Sec. \ref{sec:pre}, while the ones marked with ``2" illustrate the process of input-level regularization (Sec. \ref{sec:mixup}), and those with ``3" involve in the network-level regularization (Sec. \ref{sec:subnet}).  %(b) is a detailed exhibition of Phase MixUp in (a).
  %The gray part is related to Phase MixUp. The green part refers to the process that original target images are fed into the source model. The blue part indicates the process that the original target images enter the full target network. The green and blue parts are used for knowledge distillation in Sec. \ref{sec:pre}. The orange part presents the process which augments  input images to the full target network, and it works together with the blue part for class-consistency regularization in Sec. \ref{sec:mixup}. Besides, the pink part shows the process of augmented images passing the target subnetwork, and this part collaborates with the orange part in the network regularization in Sec. \ref{sec:subnet}.
  }
  \label{fig:framework}
  %\vspace{-4mm}
\end{figure*}

%\vspace{0.1cm}
\paragraph{Black-Box Domain Adaptation} is a subset problem of Source-Free domain adaptation and is also a very novel topic. It is more strict than Source-Free because the models trained on source will be put into a black box and \textit{only the outputs of these models can be used during model adaptation}. \cite{zhang2021unsupervised} first states this setting completely and emphasizes the purification of noisy labels in the adaptation. \cite{liang2021distill} introduces knowledge distillation to the black-box problem, which largely improves the performance. %Besides, there exist some works that apply Black-Box domain adaptation for human pose estimation like \cite{Essich_2023_WACV} or for medical image segmentation like \cite{liu2022unsupervised}. Here we also tackle this type of adaptation. 
All the details are the same as the Source-Free setting except that the source model strictly follows the Black-Box rule as its weights are completely invisible while outputs are accessible. 

%\vspace{-0.3cm}
\section{Methodology}
Our proposed method RAIN tackles the Black-Box domain adaptation from the perspective of regularization (both input-level and network-level). The overall framework is presented in Fig. ~\ref{fig:framework}a. In the following subsections, we elaborate on the key components of the framework.
\iffalse
\begin{figure}[htp]
  %\vspace{-0.1cm}
  \centering  \includegraphics[width=0.97\linewidth]{img/frame_003.pdf}
  %\vspace{-5pt}
  \caption{\small{Overview of our proposed framework (Best viewed in color). The gray part is related to Phase MixUp. The green part refers to the process that original target images are fed into the source model. The blue part indicates the process that the original target images enter the full target network. The green and blue parts are used for knowledge distillation in Sec 3.1. The orange part presents the process that augmented images come to the full target network, and it works together with the blue part for class-consistency regularization in Sec 3.2. Besides, the pink part shows the process of augmented images passing the target subnetwork, and this part collaborates with the orange part in the network regularization in Sec 3.3.\QP{Update.}}}
  \label{fig:framework}
  %\vspace{-5mm}
\end{figure}
\fi

%\vspace{-0.1cm}
\subsection{Preliminary}\label{sec:pre}
For a typical domain adaptation problem, we have a source domain dataset $\mathcal{S} = \{(x^{s}_{i},y^{s}_{i})\}^{n_{s}}_{i=1}$ with $n_{s}$ labeled samples. The target domain dataset $\mathcal{T} = \{x^{t}_{i}\}^{n_{t}}_{i=1}$ includes $n_{t}$ unlabeled samples, which shares the same label space $\mathcal{D} = \{1,2,\cdots,K\}$ with source but lie in different distributions, where $K$ is the number of classes. The goal of domain adaptation is to seek the best target model $f_{t}$ with the help of source model as $f_{s}$. 

For the Black-Box paradigm, the learning starts with the supervised learning on source as $\mathcal{L}^{s} = - \mathbb{E}_{(x^{s}_{i},y^{s}_{i}) \in \mathcal{S}} \sum_{k \in D}l^{s}_{i}\log f_{s}(x^{s}_{i})$ with label smoothing \cite{muller2019does}: $l^{s}_{i} = \alpha / K + (1-\alpha)y^{s}_{i}$, where $\alpha$ is smoothing parameter empirically set to $0.1$. Moreover, the source model's details are completely unseen except for the model's outputs. In this case, knowledge distillation \cite{hinton2015distilling} is applied to transfer from source to target:

\begin{equation}
\label{eq:ps}
    \mathcal{L}_{kd} = \mathrm{D_{KL}}(\hat{y}_{i}^{t}||f_{t}(x^{t}_{i})),
\end{equation}
where $\mathrm{D_{KL}}(\cdot)$ denotes Kullback-Leibler (KL) Divergence and $\hat{y}_{i}^{t}$ is the pseudo-label. Now we explain how to obtain $\hat{y}_{i}^{t}$. Assume that $q = \argmax f_{s}(x_{i}^{t})$, then we deduce the smooth pseudo label $l^{t}_{i} = \alpha' / K + (1-\alpha')q$, where $\alpha'$ is smoothing parameter empirically set to $0.1$. Based on this, the pseudo label $\hat{y}_{i}^{t}$ can be represented as $\hat{y}_{i}^{t} = \eta l^{t}_{i} + (1-\eta)f_{t}(x_{i}^{t})$, where $\eta$ is a momentum hyperparameter set as 0.6. 

\subsection{Enhancing Input Regularization via Phase MixUp}\label{sec:mixup}

Conventional MixUp is a very popular input- and label-level regularization technique in domain adaptation \cite{wu2020dual,liang2021distill}, whose goal is to enhance class-wise consistency and linear behavior, thus helping the model learn better representations on target domain. Nevertheless, there exist noises in the pseudo-labels applied to MixUp, which are harmful to the adaptation process, and that's why we propose an input-level (only) regularization method here. %What's more, conventional MixUp cannot highlight the task-relevant objects when the image has a complicated background, and it will lead to suboptimal performance since the core objects may be ignored and the model connects the object with the background instead. %From the top of Fig. \ref{fig:compare}, we can observe such a drawback. We want to blend the ``hammer" object into the ``television" image. However, there exists a very strong dark shadow in the ``hammer" image. %The ``television" sample here contains a relatively complex background with prairie, tree, and sky. 
%This background element has negative impacts on the MixUp effect as it outweighs the core object ``television". In this case, the core object ``television" is not focused and it is difficult for the model to build connections between the ``hammer" and the ``television". When $\mu = 0.75$, the ``television" disappears and only the ``hammer" and the dark shadow remain. Therefore, we propose Phase MixUp to highlight the key objects and reduce the background interference simultaneously. The bottom part of Fig. \ref{fig:compare} illustrates the effect of Phase MixUp, where the two core objects are highlighted. 
Next details of the proposed Phase MixUp process are presented. We begin by introducing the standard format of the Fourier transform. Assume a target sample $x_{i}^{t} \in \mathbb{R}^{C\times H\times W}$, where $C$, $H$ and $W$ correspond to channel numbers, height and width. We transfer it from spatial space to frequency space and then decompose its frequency spectrum as amplitude and phase:
\begin{equation}
\label{eq:ft}
\footnotesize
    \mathcal{F}_{i}^{t} = \sum_{h=0}^{H-1} \sum_{w=0}^{W-1}x_{i}^{t}\exp[{-j2\pi(\frac{h}{H}u + \frac{w}{W}v)}] = \mathcal{A}_{i}^{t} \exp{(\mathcal{P}_{i}^{t})}.
\end{equation} Here we ensure the one-to-one correspondence between channels of different spaces. Here $x_{i}^{t}$ is a spatial image representation based on image pixel $(h,w)$, and $\mathcal{F}_{i}^{t}$ is a frequency image representation based on frequency spectrum unit $(u,v)$. $\mathcal{A}_{i}^{t}$ is the amplitude spectrum and $\mathcal{P}_{i}^{t}$ is the phase spectrum of the target sample $x_{i}^{t}$. According to \cite{yang2020fda,liu2021feddg}, the amplitude spectrum reflects the low-level distributions like the style, \emph{and the high-level semantics like object shape is stored in the phase spectrum.} Since our task here is domain adaptation for object recognition, we hope the mixup procedure focuses more on the key objects rather than the background information. Hence, we interpolate between phase spectra as:
\begin{equation}
\label{eq:mixup}
    \mathcal{P}_{mix}^{t} = \mu \mathcal{P}_{j}^{t}  + (1 - \mu) \mathcal{P}_{i}^{t},
\end{equation} 
where $\mathcal{P}_{j}^{t}$ is the phase spectrum from a randomly-selected target sample $x_{j}^{t}$ as $\mathcal{F}(x_{j}^{t}) = \mathcal{A}_{j}^{t} \exp{(\mathcal{P}_{j}^{t})}$, and $\mu$ is sampled from a Beta distribution as $\mathbf{Beta}(0.3, 0.3)$. %Here Phase MixUp serves as a small perturbation because, with appropriate disturbances, the target representations can be more robust and generalizable. And that is the reason why we set the Beta distribution's parameter to a small value of 0.3. In such a case, the disturbance is controllable and our augmented outputs can be consistent with the original outputs. 
After that, the Phase MixUp augmented sample produced by inverse Fourier Transform is:
\begin{equation}
\footnotesize{
\begin{multlined}
    x_{mix}^{t} = \mathcal{F}^{-1}(\mathcal{A}_{i}^{t}, \mathcal{P}_{mix}^{t}) \\= \frac{1}{HW}\sum_{u=0}^{H-1} \sum_{v=0}^{W-1} \mathcal{A}_{i}^{t} \exp{(\mathcal{P}_{mix}^{t})}\exp[{-j2\pi(\frac{u}{H}h + \frac{v}{W}w)}].
\end{multlined}
}
\end{equation} The Phase MixUp procedure is depicted in Fig. \ref{fig:framework}b. After obtaining the synthesized sample, we can enhance class consistency by comparing the outputs of the original and synthesized samples. Here $\ell_1$-norm is utilized to compute the Phase MixUp loss as:
\begin{equation}
\label{eq:mix}
      \mathcal{L}_{pm} =  \mathbb{E}_{x_{i}^{t} \in T}\left\Vert f_{t}(x_{i}^{t}) - f_{t}(x_{mix}^{t}) \right\Vert_{\ell_1}.
\end{equation}Phase MixUp is different from conventional MixUp. First, our Phase MixUp is conducted on the phase spectra related to core objects, not the whole image. Moreover, conventional MixUp operates on both input- and label-level, while Phase MixUp is an input-level augmentation.

\begin{figure}
%\vspace{-5mm}
  \centering
  \includegraphics[width=0.97\linewidth]{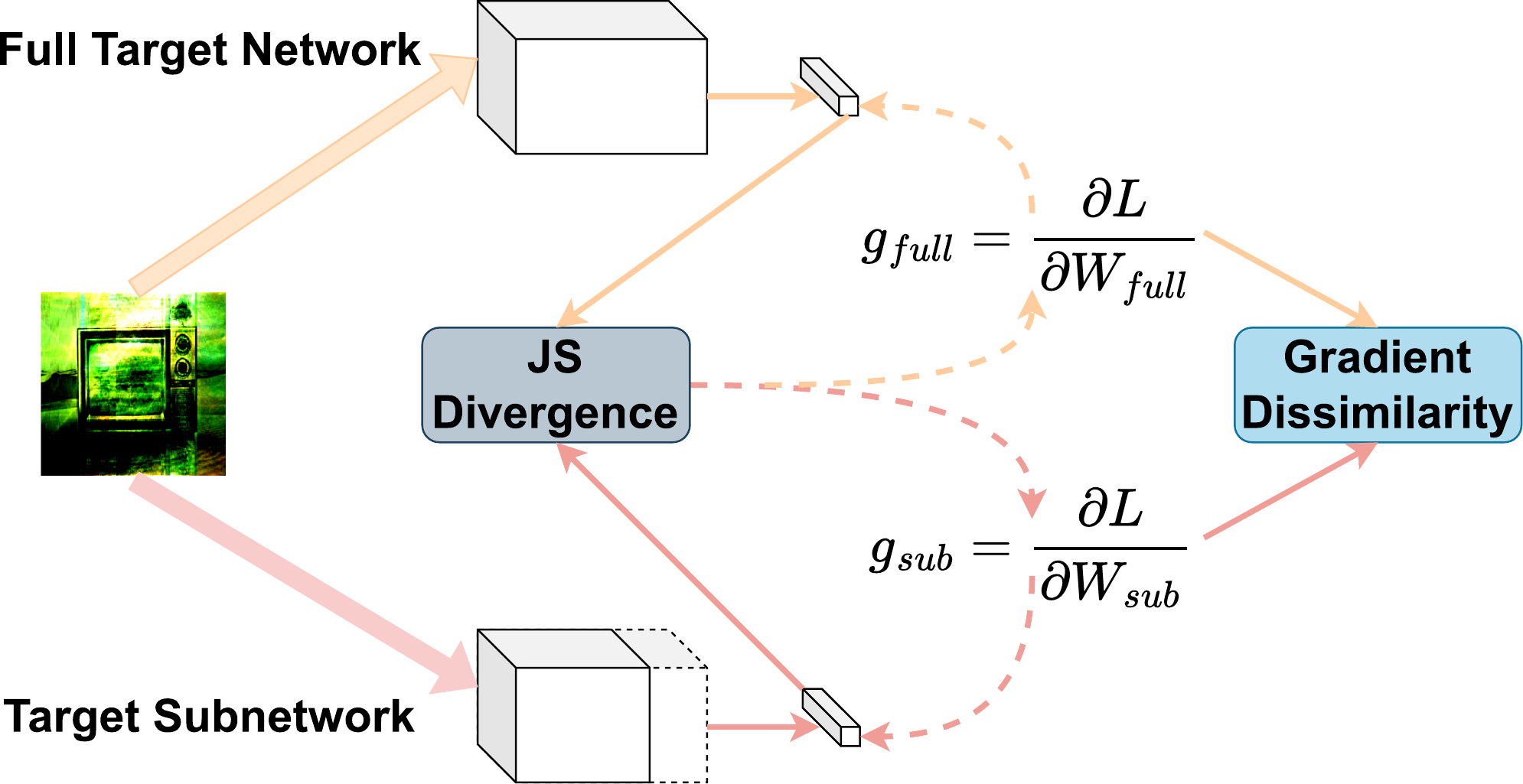}
  %\vspace{-5pt}
  \caption{The training process of Subnetwork Distillation (Best viewed in color.) The orange arrows associate with the full target network adaptation and the pink arrows correspond to the target subnetwork adaptation. During the optimization of JS Divergence and Gradient Dissimilarity, the model obtains knowledge of diverse target representations that benefit the generalization.}
  \label{fig:mutual}
%\vspace{-10pt}
\end{figure}

\subsection{Encouraging Network Regularization via Subnetwork Distillation}\label{sec:subnet}

During the procedure of knowledge transfer from source models to target models with knowledge distillation (Eq. \ref{eq:ps}), overfitting on source information is an obvious side effect. Especially in the Black-Box setting, only the source model's outputs are visible while the source model's weights are completely unseen. In other words, the careful calibration of the target network's weights in the Source-Free setting cannot be achieved here. Hence, it is necessary to propose a specific network-based regularization method.%, enforcing regularizations and calibrations on target models to alleviate overfitting on source information. %What's more, previous works only consider the knowledge transfer from the source model to the target model. \textit{The knowledge transfer inside the target model is overlooked.}
To this end, we propose the Subnetwork Distillation approach to domain adaptation, which utilizes the self-knowledge transfer from the target subnetwork to the full target network with distinctive knowledge. We hope this structure can assist the model to obtain more target information from diverse representations far from the support of source, thus overcoming overfitting on source. %which trains the target network and its subnetworks using multiple input (image) resolutions to learn diverse representations. Here we hope that this structure can assist the model to obtain more target information from diverse representations, thus overcoming overfitting on source. 
We denote the full target network's weights as $W_{full}$ and the target subnetwork as $W_{sub}$. If the complete network can be represented as $W_{full} = W_{0:1}$, then by slimming the network width with a ratio $\alpha \in (0,1]$, the subnetwork can be generated as $W_{sub} = W_{0:\alpha}$, i.e., a subnetwork with $W_{0:\alpha}$ means selecting the first $\alpha \times 100\%$ weights of each layer of the full network. %As we aim to avoid overfitting on source, it is meaningful to provide various representations from target as diverse receptive fields, by leveraging inputs with multiple resolutions. The resolution set $\mathcal{R} = \{r_{1}\times r_{1}, r_{2}\times r_{2},  r_{3}\times r_{3}\}$ consists of three resolutions satisfying $0 < r_{1}, r_{2}, r_{3} < 224$. $I_{r}$ requires input resolution as $r \times r \in \mathcal{R}$. Therefore, the network's output with width $\alpha$ and input size $r \times r$ is denoted as $f_{t}(x_{i}^{t}; W_{0:\alpha}; I_{r})$. 
Therefore, the network's output with width $\alpha$ is denoted as $f_{t}(x_{i}^{t}; W_{0:\alpha})$. %Our goal is to learn diverse target representations via self-knowledge transfer from the target subnetwork to the full target network, thus assisting the full network to generalize better on target. 
The Subnetwork Distillation objective is defined as
\begin{equation}
\label{eq:ml}
      \mathcal{L}_{sd} = \mathrm{D_{JS}}(f_{t}(x_{mix}^{t}; W_{sub})|| f_{t}(x_{mix}^{t}; W_{full})).
\end{equation}
After getting these outputs from the subnetwork with smaller widths, we compare them with the original network's outputs using Jensen–Shannon (JS) divergence, shown in Fig. \ref{fig:mutual}. To prevent the adverse influence on the inference with original images and full networks, here we use the images operated after Phase MixUp, which can also add perturbations to the regularization for more robust representations.%Resolutions are set as 224, 192, 160, and 128, while subnetwork width is set as $\alpha=0.9\times$. 

Now we provide a theoretical analysis to illustrate why the JS divergence can benefit the adaptation process. Assume that Phase MixUp is a mapping function $f_{PM}:x\rightarrow z$, and the following neural network is $f_{network}:z\rightarrow y$. Here we have three types of networks, the source network $f_{s}$, the full target network $f_{t}$, and the target subnetwork $f_{sub}$. Besides, we make two assumptions that are intuitive to understand:
\begin{assumption}
    Based on observed latent variables $z$ and an empirical predictor $\hat{p}$, $-\log\hat{p}(y|z)$ can be bounded by $C$, which is a constant.
\end{assumption}

\begin{assumption}
The target subnetwork is superior to the source network on the target datasets. Mathematically, 
\begin{equation}
    p_{s}(y,z)\log\hat{p}(y|z) \leq p_{sub}(y,z)\log\hat{p}(y|z).
\end{equation}
\end{assumption}

What's more, there is a lemma that assists our main conclusion Theorem \ref{theorem}:
\begin{lemma}
    The source loss and target loss are bounded by joint distributions and empirical predictions as:
    \begin{equation}
       l_{s} \leq \mathbb{E}_{p_{s}(y,z)}[-\log\hat{p}(y|z)], ~~~
         l_{t} \leq \mathbb{E}_{p_{t}(y,z)}[-\log\hat{p}(y|z)].
     \end{equation}
\end{lemma}

On the basis of the proposed assumptions and lemma, we conclude the bound for the target loss as:

\begin{theorem}
\label{theorem}
The target loss can be bounded by source loss and the JS divergence between the outputs of the full target network and target subnetwork as:
\begin{equation}
    l_{t} \leq l_{s} + C\sqrt{2\mathrm{D_{JS}}(p_{t}(y,z)||p_{sub}(y,z))},
\end{equation}
\end{theorem}

$\mathrm{D_{JS}}(p_{t}(y,z)||p_{sub}(y,z))$ corresponds to Eq. \ref{eq:ml}, which proves that the optimization of the JS divergence can benefit the adaptation process on the target domain. %More detailed proof is presented in the \textbf{Appendix}.

%The success of Eq. \ref{eq:ml} relies on learning diverse representations from the subnetwork, but 
There exist two extremes that hinder our goal. One is at the \emph{beginning} of adaptation, when the subnetwork is very different from the full network, but it owns much greater errors than the full network, and the knowledge transfer is negative to the adaptation. The other is at the \emph{end} of adaptation, subnetwork is very similar to the full network so that not enough knowledge can be transferred from the subnetwork to the full network. In such a case, the Subnetwork Distillation is meaningless. %In such a case, the self-knowledge distillation is meaningless.%One is the subnetwork is very similar to the full network so that not enough knowledge can be transferred from the subnetwork to the full network. %In such a case, the self-knowledge distillation is meaningless. 
%This often happens when the model becomes stable. The other is the subnetwork is very different from the full network, but it owns much greater errors than the full network, and the knowledge transfer is negative to the adaptation. 
%This often appears at the beginning of the adaptation. 
Provided that the gradient of the full network is $g_{full} = \frac{\partial \mathcal{L}_{sd}}{\partial W_{full}}$, and the gradient of the subnetwork is $g_{sub} = \frac{\partial \mathcal{L}_{sd}}{\partial W_{sub}}$, then we propose a weighted gradient discrepancy loss to balance this two extremes:%gradient discrepancy is applied to quantize the knowledge obtained from the full network and subnetwork. Provided that the gradient of the full network is $g_{full} = \frac{\partial \mathcal{L}_{sr}}{\partial W_{full}}$, then the gradient of subnetwork is $g_{sub} = \frac{\partial \mathcal{L}_{sr}}{\partial W_{sub}}$. Based on these, we propose a weighted gradient discrepancy loss:
%\vspace{-0.2cm}
\begin{equation}
\footnotesize{
\label{eq:grad}
    \mathcal{L}_{wg} = (1 + \exp{(-\mathbf{Entropy}(f_{t}(x_{i}^{t}; W_{sub})))})\frac{g_{full}^{T} g_{sub}}{\left\Vert g_{full} \right\Vert_2 \left\Vert g_{sub} \right\Vert_2}.
}
\end{equation} Here the left term is the weight and the right term is the cosine similarity between $g_{sub}$ and $g_{full}$. By minimizing the right term, the discrepancy between two gradients is enlarged, which provides a disturbance and guarantees that the full network and subnetwork learn divergent knowledge. The left term offers constraints on the gradient dissimilarity so that if the subnetwork doesn't learn distributions that are confident enough, a smaller weight will be assigned, and vice versa. %The reason why we use a weighted loss is that at the beginning of the adaptation, the subnetwork is very unstable, which leads to the great dissimilarity between the gradients. However, this kind of dissimilarity should be discouraged as the subnetwork doesn't learn the right distributions, and the knowledge distilled from the subnetwork will mislead the full network. Thus, we assign a smaller weight when the confidence of the subnetwork's outputs is smaller and vice versa. 

\subsection{Overall Objectives}

Based on the above discussion, we conclude the final objective for Black-Box domain adaptation:
\begin{equation}
\label{eq:total-bbda}
    \mathcal{L}_{bb} = \mathcal{L}_{kd} + \beta \mathcal{L}_{pm} + \gamma \mathcal{L}_{sd} + \theta \mathcal{L}_{wg},
\end{equation} 
where $\beta, \gamma$, and $\theta$ are the trade-off hyperparameters for corresponding loss functions.  
\begin{table*}[!ht]
%\vspace{-0.6cm}
    %\setlength{\abovecaptionskip}{0pt}
    \centering
    %\vspace{-0.3cm}
    
   %%%%%% - Office 31 ================== 
    %\setlength\tabcolsep{20.5pt} 
    \renewcommand{\arraystretch}{1.0}% Tighter
    \resizebox{0.99\linewidth}{19mm}{%   
    \begin{tabular}{Sc Sc Sc Sc Sc Sc Sc a}
         \toprule
          \textbf{\textcolor{red}{(a) Office-31}} &  A$\rightarrow$ D &  A$\rightarrow$ W  &  D$\rightarrow$ A & D$\rightarrow$ W & W$\rightarrow$ A &  W$\rightarrow$ D & \textbf{Avg.}\\
          
         \hline
         {Source-only} & 79.9/88.2 & 76.6/89.2 & 56.4/74.5 & 92.8/97.2 & 60.9/77.2 & 98.5/99.3 & 77.5/87.6\\
         \hline
         
         {NLL-OT} & 88.8/91.3 & 85.5/91.4 & 64.6/76.4 & 95.1/97.2 & 66.7/78.2 & 98.7/99.4 & 83.2/89.0\\
         {NLL-KL} & 89.4/91.7 & 86.8/91.8 & 65.1/76.3 & 94.8/97.2 & 67.1/78.4 & 98.7/99.0 & 83.6/89.1\\
         {HD-SHOT} & 86.5/88.9 & 83.1/99.9 & 66.1/75.3 & 95.1/97.7 & 68.9/77.7 & 98.1/99.5 & 83.0/88.3\\
         {SD-SHOT} & 89.2/91.6 & 83.7/92.8 & 67.9/77.8 & 95.3/\textbf{98.7} & 71.1/78.5 & 97.1/\textbf{99.7} & 84.1/89.8\\
         {DINE} & 91.6/94.2 & 86.8/94.6 & 72.2/80.7 & 96.2/98.8 & 73.3/81.5 & 98.6/99.5 & 86.4/91.6\\
         {DINE-full} & 91.7/95.5 & 87.5/94.8 & 72.9/81.2 & 96.3/98.5 & 73.7/82.0 & 98.5/\textbf{99.7} & 86.7/91.9\\
         \hline
         {\textbf{RAIN}} & \Centerstack{\textbf{93.8}$\pm0.20$/\\ \textbf{96.2}$\pm0.20$ }& \Centerstack{\textbf{88.8}$\pm0.10$/\\\textbf{95.7}$\pm0.10$} & \Centerstack{\textbf{75.5}$\pm0.14$/\\ \textbf{83.6}$\pm0.07$} & \Centerstack{\textbf{96.8}$\pm0.10$/\\98.6$\pm0.10$ } & \Centerstack{\textbf{76.7}$\pm0.14$/\\\textbf{84.1}$\pm0.21$} & \Centerstack{\textbf{99.5}$\pm0.10$/\\\textbf{99.7}$\pm0.20$} & \Centerstack{\textbf{88.5}$\pm0.10$/\\\textbf{93.0}$\pm0.10$}\\
         \bottomrule
    \end{tabular}} \par\vskip-1.4pt
    
 %%%%%% ===================Office Home ===========================   
    
    %\setlength\tabcolsep{2.0pt} 
    %\renewcommand{\arraystretch}{0.8}% Tighter
    \resizebox{0.99\linewidth}{20mm}{%
    \begin{tabular}{Sc Sc Sc Sc Sc Sc Sc Sc Sc Sc Sc Sc Sc a}
         \toprule
          \textbf{\textcolor{amethyst}{(b) Office-Home}} &  Ar$\rightarrow$ Cl &  Ar$\rightarrow$ Pr  & Ar$\rightarrow$ Rw &   Cl$\rightarrow$ Ar & Cl$\rightarrow$ Pr & Cl$\rightarrow$ Rw & Pr$\rightarrow$ Ar &  Pr$\rightarrow$ Cl & Pr$\rightarrow$ Rw & Rw$\rightarrow$ Ar &  Rw$\rightarrow$ Cl & Rw$\rightarrow$ Pr & \textbf{Avg.} \\
          \hline
         {Source-only} & 44.1/54.5 & 66.9/83.2 & 74.2/87.2 & 54.5/78.0 & 63.3/83.8 & 66.1/86.1 & 52.8/74.5 & 41.2/49.7 & 73.2/87.4 & 66.1/78.6 & 46.7/52.6 & 77.5/86.2 & 60.6/75.1\\ 
         \hline
         
         {NLL-OT} & 49.1/58.8 & 71.7/84.4 & 77.3/87.6 & 60.2/78.2 & 68.7/84.7 & 73.1/86.7 & 57.0/76.0 & 46.5/54.0 & 76.8/88.0 & 67.1/79.7 & 52.3/57.2 & 79.5/87.2 & 64.9/76.9\\
         {NLL-KL} & 49.0/59.5 & 71.5/84.3 & 77.1/87.6 & 59.0/77.4 & 68.7/84.8 & 72.9/86.8 & 56.4/75.1 & 46.9/54.9 & 76.6/88.0 & 66.2/79.0 & 52.3/57.9 & 79.1/87.2 & 64.6/76.9\\
         {HD-SHOT} & 48.6/57.2 & 72.8/84.2 & 77.0/87.3 & 60.7/78.4 & 70.0/84.9 & 73.2/86.4 & 56.6/74.8 & 47.0/56.0 & 76.7/87.6 & 67.5/78.9 & 52.6/57.5 & 80.2/87.0 & 65.3/76.7\\
         {SD-SHOT} & 50.1/59.4 & 75.0/85.2 & 78.8/87.8 & 63.2/79.6 & 72.9/86.6 & 76.4/87.1 & 60.0/76.4 & 48.0/58.3 & 79.4/87.8 & 69.2/80.0 & 54.2/59.5 & 81.6/87.9 & 67.4/78.0\\
         {DINE} & 52.2/64.9 & 78.4/87.4 & 81.3/88.8 & 65.3/80.5 & 76.6/\textbf{89.6} & 78.7/87.8 & 62.7/79.0 & 49.6/\textbf{62.9} & 82.2/89.1 & 69.8/81.5 & 55.8/64.6 & 84.2/\textbf{90.0} & 69.7/80.5\\
         {DINE-full} & 54.2/64.4 & 77.9/87.9 & 81.6/89.0 & 65.9/80.9 & 77.7/\textbf{89.6} & 79.9/88.7 & 64.1/79.6 & 50.5/62.5 & 82.1/89.4 & 71.1/81.7 & 58.0/65.2 & 84.3/89.7 & 70.6/80.7\\
         \hline
         {\textbf{RAIN}} & \Centerstack{\textbf{57.0}$\pm0.06$/\\\textbf{66.3}$\pm0.12$} & \Centerstack{\textbf{79.7}$\pm0.08$/\\\textbf{88.8}$\pm0.12$} & \Centerstack{\textbf{82.8}$\pm0.04$/\\\textbf{90.1}$\pm0.08$} & \Centerstack{\textbf{67.9}$\pm0.04$/\\\textbf{82.0}$\pm0.04$} & \Centerstack{\textbf{79.5}$\pm0.06$/\\89.5$\pm0.12$} & \Centerstack{\textbf{81.2}$\pm0.04$/\\\textbf{89.2}$\pm0.04$} & \Centerstack{\textbf{67.7}$\pm0.04$/\\\textbf{80.7}$\pm0.08$} & \Centerstack{\textbf{53.2}$\pm0.08$/\\\textbf{62.9}$\pm0.12$} & \Centerstack{\textbf{84.6}$\pm0.12$/\\\textbf{90.6}$\pm0.12$} & \Centerstack{\textbf{73.3}$\pm0.04$/\\\textbf{82.9}$\pm0.08$} & \Centerstack{\textbf{59.6}$\pm0.06$/\\\textbf{65.8}$\pm0.06$} & \Centerstack{\textbf{85.6}$\pm0.08$/\\89.8$\pm0.04$} & \Centerstack{\textbf{73.0}$\pm0.08$/\\\textbf{81.6}$\pm0.16$}\\
         \bottomrule
         
    \end{tabular}} \par\vskip-1.4pt
    
    %%%%%%%%% VisDA - C =====================
    
    \resizebox{0.99\linewidth}{19.0mm}{%
    \begin{tabular}{Sc Sc Sc Sc Sc Sc Sc Sc Sc Sc Sc Sc Sc a}
         \toprule
         \textbf{\textcolor{blue}{(c) VisDA-C}} & plane &  bcycl  & bus &  car & horse & knife & mcycl &  person & plant & sktbrd &  train & truck & \textbf{Avg.} \\%\\[4ex]
         \hline
         {Source-only} & 64.3/97.0 & 24.6/56.2 & 47.9/81.0 & 75.3/74.4 & 69.6/91.8 & 8.5/52.0 & 79.0/92.5 & 31.6/10.1 & 64.4/73.4 & 31.0/92.7 & 81.4/97.0 & 9.2/17.5 & 48.9/69.6\\
         \hline
         {NLL-OT} & 82.6/\textbf{97.8} & 84.1/90.8 & 76.2/81.9 & 44.8/49.7 & 90.8/95.7 & 39.1/93.5 & 76.7/85.2 & 72.0/45.4 & 82.6/88.9 & 81.2/96.6 & 82.7/91.2 & 50.6/54.4 & 72.0/80.9\\
         {NLL-KL} & 82.7/97.6 & 83.4/91.1 & 76.7/82.1 & 44.9/49.2 & 90.9/95.8 & 38.5/93.5 & 78.4/\textbf{86.2} & 71.6/44.6 & 82.4/89.0 & 80.3/96.4 & 82.9/91.4 & 50.4/54.8 & 71.9/81.0\\
         {HD-SHOT} & 75.8/96.7 & 85.8/91.7 & 78.0/81.8 & 43.1/48.4 & 92.0/95.1 & 41.0/\textbf{98.5} & 79.9/83.1 & 78.1/60.1 & 84.2/\textbf{92.2} & 86.4/87.7 & 81.0/88.4 & \textbf{65.5}/\textbf{65.3} & 74.2/82.4\\
         {SD-SHOT} & 79.1/96.3 & 85.8/91.1 & 77.2/80.3 & 43.4/46.4 & 91.6/93.9 & 41.0/98.2 & 80.0/81.5 & 78.3/58.6 & 84.7/90.9 & 86.8/85.5 & 81.1/88.0 & 65.1/63.8 & 74.5/81.2\\
         {DINE} & 81.4/96.6 & 86.7/91.9 & 77.9/83.1 & 55.1/58.2 & 92.2/95.3 & 34.6/97.8 & 80.8/85.0 & 79.9/73.6 & 87.3/91.9 & 87.9/94.9 & 84.3/92.2 & 58.7/60.7 & 75.6/85.1\\
         {DINE-full} & 95.3/96.6 & 85.9/91.9 & 80.1/82.9 & 53.4/57.9 & 93.0/95.4 & 37.7/97.8 & 80.7/84.5 & 79.2/73.1 & 86.3/91.7 & 89.9/95.1 & 85.7/92.0 & 60.4/60.9 & 77.3/85.0\\
         \hline
         {\textbf{RAIN}} & \Centerstack{\textbf{96.6}$\pm0.09$/\\97.7$\pm0.09$} & \Centerstack{\textbf{86.8}$\pm0.09$/\\\textbf{92.8}$\pm0.09$} & \Centerstack{\textbf{83.0}$\pm0.06$/\\\textbf{86.2}$\pm0.12$} & \Centerstack{\textbf{70.9}$\pm0.04$/\\\textbf{72.3}$\pm0.08$} & \Centerstack{\textbf{94.5}$\pm0.08$/\\\textbf{96.5}$\pm0.04$} & \Centerstack{\textbf{81.8}$\pm0.10$/\\98.0$\pm0.10$} & \Centerstack{\textbf{84.2}$\pm0.06$/\\\textbf{86.2}$\pm0.12$} & \Centerstack{\textbf{83.6}$\pm0.09$/\\\textbf{83.2}$\pm0.09$} & \Centerstack{\textbf{90.9}$\pm0.08$/\\92.1$\pm0.16$} & \Centerstack{\textbf{89.5}$\pm0.08$/\\\textbf{96.9}$\pm0.08$} & \Centerstack{\textbf{89.4}$\pm0.06$/\\\textbf{93.3}$\pm0.06$} & \Centerstack{64.0$\pm0.08$/\\60.7$\pm0.08$} & \Centerstack{\textbf{82.7}$\pm0.09$/\\\textbf{86.6}$\pm0.09$}\\
         \bottomrule
    \end{tabular}}

    %\vspace{-0.1cm}
    \caption{Single-Source Domain Adaptation Accuracy (\%) on (a) Office-31, (b) Office-Home, and (c) VisDA-C. In each cell, the value before the forward slash $/$ origins from ResNet-based source model, and the one after $/$ relies on ViT-based source model.}
    \label{tab:ss-bbda}
\end{table*}

\begin{table*}[!ht]
\centering
%\vspace{-0.3cm}
%\setlength\tabcolsep{18.5pt} 
\renewcommand{\arraystretch}{0.9}
\resizebox{0.99\linewidth}{17.5mm}{
    \begin{tabular}{Sc|Sc Sc Sc a|Sc Sc Sc a|Sc Sc Sc Sc a}
    \toprule
         Dataset & \multicolumn{4}{c|}{\textcolor{red}{Office-31}} & \multicolumn{4}{c|}{\textcolor{blue}{Image-CLEF}} & \multicolumn{5}{c}{\textcolor{amethyst}{Office-Home}}  \\ 
         \hline
         Method & $\rightarrow$A & $\rightarrow$D & $\rightarrow$W & \textbf{Avg.} & $\rightarrow$C & $\rightarrow$I & $\rightarrow$P & \textbf{Avg.} & $\rightarrow$Ar & $\rightarrow$Cl & $\rightarrow$Pr & $\rightarrow$Rw & \textbf{Avg.}\\
         \hline
         No Adapt. & 64.5/77.2 & 82.3/88.2 & 80.7/89.2 & 75.8/84.9 & 92.1/95.3 & 87.4/90.2 & 72.4/72.0 & 84.0/85.9 & 54.9/74.5 & 49.9/54.5 & 69.6/83.2 & 76.7/87.2 & 62.8/74.8\\
         SD-DECISION & 66.6/80.0 & 87.3/90.4 & 85.7/95.9 & 80.0/88.8 & 93.5/95.0 & 89.6/91.8 & 74.1/76.6 & 85.7/87.8 & 62.5/77.2 & 51.9/55.8 & 72.3/85.3 & 80.4/88.8 & 66.8/76.8   \\
         DINE w/o FT & 69.2/80.7 & 98.6/98.4 & 96.9/97.1 & 88.3/92.1 & 96.2/97.2 & 91.4/96.6 & 78.3/80.9 & 88.6/91.6 & 70.8/82.4 & 57.1/61.0 & 80.9/88.6 & 82.1/90.8 & 72.7/80.7  \\
         DINE & 76.8/82.4 & 99.2/99.2 & 98.4/98.4 & 91.5/93.4 & 98.0/97.8 & 93.4/96.6 & 80.2/81.3 & 90.5/91.9 & 74.8/83.6 & 64.1/67.0 & 85.0/90.9 & 84.6/91.9 & 77.1/83.3\\
         DINE-full & 77.1/81.4 & 99.2/99.0 & 98.2/98.5 & 91.5/93.0 & 97.8/97.8 & 93.0/96.4 & 79.7/81.4 & 90.2/91.9 & 74.9/83.4 & 62.6/65.2 & 84.6/90.3 & 84.7/91.5 & 76.7/82.6\\
         \hline
         \textbf{RAIN} & \Centerstack{\textbf{79.8}$\pm$0.14/\\\textbf{84.5}$\pm$0.14}&\Centerstack{\textbf{99.8}$\pm0.20$/\\\textbf{99.2}$\pm0.20$}&\Centerstack{\textbf{99.0}$\pm0.10$/\\\textbf{98.6}$\pm0.10$}&\Centerstack{\textbf{92.9}$\pm0.10$/\\\textbf{94.1}$\pm0.10$}&\Centerstack{\textbf{98.4}$\pm0.17$/\\\textbf{98.0}$\pm0.17$}&\Centerstack{\textbf{94.2}$\pm0.17$/\\\textbf{96.6}$\pm0.17$}&\Centerstack{\textbf{82.0}$\pm0.17$/\\\textbf{83.2}$\pm0.17$}&\Centerstack{\textbf{91.5}$\pm0.17$/\\\textbf{92.6}$\pm0.17$}&\Centerstack{\textbf{76.0}$\pm0.08$/\\\textbf{84.0}$\pm0.08$}&\Centerstack{\textbf{65.6}$\pm0.12$/\\\textbf{68.5}$\pm0.12$}&\Centerstack{\textbf{85.8}$\pm0.12$/\\\textbf{92.0}$\pm0.12$}&\Centerstack{\textbf{84.8}$\pm0.08$/\\\textbf{93.0}$\pm0.08$}&\Centerstack{\textbf{78.1}$\pm0.16$/\\\textbf{84.4}$\pm0.16$}\\
         \bottomrule
    \end{tabular}}
%\vspace{-0.2cm}
\caption{Multi-Source Domain Adaptation Accuracy (\%) on Office-31, Image-CLEF, and Office-Home.  In each cell, the value before the forward slash $/$ origins from ResNet-based source model, and the one after $/$ relies on ViT-based source model.}
\label{tab:ms-bbda}
\end{table*}

\section{Experiments}
\paragraph{Datasets. }We use four popular benchmark datasets for evaluation. \textbf{Office-31} \cite{saenko2010adapting} has three domains as Amazon, Webcam, and DSLR with 31 classes and 4,652 images. \textbf{Image-CLEF} \cite{long2017deep} is a relatively small dataset with three domains, and each domain includes 12 classes and 600 images. \textbf{Office-Home} \cite{venkateswara2017deep} is a medium-size dataset, containing four domains as Art, Clipart, Product, and Real World. Each domain includes 65 classes and the total number of images is 15,500. \textbf{VisDA-C} \cite{peng2017visda} is the most challenging dataset among the four, with 152,000 synthesized images serving as the source domain and 55,000 real images serving as the target domain, each with 12 classes. %Office-31, Office-Home and VisDA-C are used for Single-Source adaptation, while Office-31, Office-Home and Image-CLEF are utilized for Multi-Source adaptation.

\paragraph{Model Architecture.} We adopt ResNet-50 \cite{he2016deep} as the backbone for Office-31 and Office-Home, and ResNet-101 for VisDA-C. To facilitate fair comparisons, we follow the protocols from DINE \cite{liang2021distill}, replacing the last layer of the target network with a pipeline as a fully-connected layer, batch normalization layer, fully-connected layer, and weight normalization layer. As for the pretrained source model, ResNet and Vision Transformer (ViT) \cite{dosovitskiy2020image} are applied.%As for the model trained on source, it is almost the same as the original ResNet, with only the number of classes changed.

\paragraph{Implementation.} We set the batch size to 64 and adopt SGD \cite{ruder2016overview} as the optimizer, with a momentum of 0.9 and a weight decay of 1e-3. For Office-31 and Office-Home, the learning rate is set as 1e-3 for the convolutional layers and 1e-2 for the rest. For VisDA-C, we choose 1e-4 for the convolutional layers and 1e-3 for the rest. The learning rate scheduler is the same as \cite{liang2020we}, i.e., a polynomial annealing strategy. Label smoothing \cite{muller2019does} is used on the leverage of source client, with 100 epochs for all the tasks. For the training procedure on target client, we train 30 epochs for all the tasks. In the evaluation stage, all results are obtained by averaging three random runs. For the hyper-parameters in Eq. \ref{eq:total-bbda}, we set $\beta =1.2$, $\gamma =0.6$, and $\theta =0.3$. PyTorch \cite{paszke2019pytorch} is used for the implementation. For the proposed Subnetwork Distillation (Sec. 3.3), the subnetwork width is set as 0.84$\times$.% to avoid large gap from the full network. 
When it comes to inference, we only use the full target network. Training is conducted on an NVIDIA RTX A5000 GPU. 

%\vspace{-0.55mm}
\paragraph{Baselines.} Several state-of-the-art baselines are selected for comparison. For Single-Source adaptation, we compare our method with NLL-OT \cite{YM-2020Self-labelling}, NLL-KL \cite{zhang2021unsupervised}, HD-SHOT \cite{liang2020we}, SD-SHOT \cite{liang2020we}, DINE \cite{liang2021distill}, and DINE-full \cite{liang2021distill}. For Multi-Source adaptation, we compare with SD-DECISION \cite{ahmed2021unsupervised}, DINE w/o Fine-Tune (FT) \cite{liang2021distill},  DINE \cite{liang2021distill}, and DINE-full \cite{liang2021distill}.

\subsection{Results}

%\textbf{Source-Free Adaptation.} We compare our method (FMML) with existing Source-Free approaches and the results on three datasets are shown in Table \ref{tab:sfda}. Based on the average accuracies on these datasets, our model achieves the state-of-the-art performances. Comparing with the second best approaches such as HCL on Office-31 and SHOT++ on Office-Home and VisDA-C, our method (FMML) gains improvements of 1.0\%, 1.1\%, and 0.8\% correspondingly. When examining specific tasks, our approach outperforms all the competing baselines on 4 out of 6 tasks on Office-31, 10 out of 12 tasks on Office-Home, and 5 out of 12 tasks on VisDA-C. 

We compare our method (RAIN) with existing Single-Source and Multi-Source approaches, and the results are shown in Table \ref{tab:ss-bbda} and Table \ref{tab:ms-bbda} separately. %In each cell of a table, the value before the forward slash $/$ is distilled from the ResNet-based source model, while the value after $/$ is distilled from the ViT-based source model.

For Single-Source adaptation, our model achieves state-of-the-art performances on average in three datasets. Compared with the second-best results, our method yields improvements of $1.8\%/1.1\%$ in Office-31, $2.4\%/0.9\%$ in Office-Home, and $5.4\%/1.6\%$ in VisDA-C. And in most tasks, our approach outperforms all the listed baselines. For Multi-Source adaptation, our model also obtains state-of-the-art results in three datasets, yielding average improvements of $1.4\%/1.1\%$ in Office-31, $1.3\%/0.7\%$ in Image-CLEF, and $1.4\%/1.8\%$ in Office-Home. Besides, in all the tasks of Multi-Source adaptation, RAIN outperforms all the listed baselines. 
%\noindent  \textbf{Single-Source Adaptation.} We compare our method (PMML) with the existing Black-Box approaches and the results are reported in Table \ref{tab:ss-bbda}. For the average accuracy on three datasets, our model achieves state-of-the-art performances. Comparing with the second-best results, our method yields improvements of 0.5\%, 0.7\%, and 0.4\%, respectively. As for specific tasks, our approach outperforms all the listed baselines on 4 out of 6 tasks on Office-31, 11 out of 12 tasks on Office-Home, and 5 out of 12 tasks on VisDA-C. 

\subsection{Analysis}

%\vspace{-0.1cm}
\begin{table}[!ht]%\scriptsize

%\vspace{-0.3cm}
\centering
\resizebox{0.95\linewidth}{12.5mm}{%
    \begin{tabular}{l Sc Sc Sc}
    \toprule
    \textcolor{red}{Single-Source} & A$\rightarrow$ D & Ar$\rightarrow$ Rw & Syn$\rightarrow$ Real \\
    \hline 
    w/o $\mathcal{L}_{pm}$ \& $\mathcal{L}_{sd}$ \& $\mathcal{L}_{wg}$ (baseline) & 90.2 & 79.3 & 76.0\\
    \hline
    w/ $\mathcal{L}_{pm}$ & 91.6 ($\uparrow$1.4)& 80.4 ($\uparrow$1.1) & 77.7 ($\uparrow$1.7)\\
    w/ $\mathcal{L}_{sd}$ & 90.8 ($\uparrow$0.6) & 79.7 ($\uparrow$0.4) & 77.0 ($\uparrow$1.0)\\
    w/ $\mathcal{L}_{pm}$ \& $\mathcal{L}_{sd}$ & 91.9 ($\uparrow$1.7) & 81.0 ($\uparrow$1.7) & 78.9 ($\uparrow$2.9)\\
    w/ $\mathcal{L}_{sd}$ \& $\mathcal{L}_{wg}$ & 93.1 ($\uparrow$2.9) & 81.6 ($\uparrow$2.3) & 81.7 ($\uparrow$5.7) \\
    RAIN ($\mathcal{L}_{pm}$ + $\mathcal{L}_{sd}$ + $\mathcal{L}_{wg}$) & 93.8 ($\uparrow$3.6) & 82.8 ($\uparrow$3.5) & 82.7 ($\uparrow$6.7)\\
    \bottomrule
   \end{tabular}} \par\vskip-1.4pt
\resizebox{0.95\linewidth}{12.5mm}{%
    \begin{tabular}{l Sc Sc Sc}
    \toprule
    \textcolor{blue}{Multi-Source} & $\rightarrow$A & $\rightarrow$P & $\rightarrow$Ar \\
    \hline 
    w/o $\mathcal{L}_{pm}$ \& $\mathcal{L}_{sd}$ \& $\mathcal{L}_{wg}$ (baseline) & 74.9 & 78.6 & 72.1\\
    \hline
    w/ $\mathcal{L}_{pm}$ & 76.5 ($\uparrow$1.6) & 80.2 ($\uparrow$1.6) & 73.3 ($\uparrow$1.2)\\
    w/ $\mathcal{L}_{sd}$ & 75.4 ($\uparrow$0.5)& 79.4 ($\uparrow$0.8)& 73.0 ($\uparrow$0.9)\\
    w/ $\mathcal{L}_{pm}$ \& $\mathcal{L}_{sd}$ & 77.7 ($\uparrow$2.8)& 80.8 ($\uparrow$2.2)& 74.5 ($\uparrow$2.4)\\
    w/ $\mathcal{L}_{sd}$ \& $\mathcal{L}_{wg}$ & 78.0 ($\uparrow$3.1)& 81.2 ($\uparrow$2.6)& 74.6 ($\uparrow$2.5)\\
    RAIN ($\mathcal{L}_{pm}$ + $\mathcal{L}_{sd}$ + $\mathcal{L}_{wg}$) & 79.8 ($\uparrow$4.9) & 82.0 ($\uparrow$3.4) & 76.0 ($\uparrow$4.1)\\
    \bottomrule
   \end{tabular}}
%\vspace{-0.2cm}
\caption{Ablation Study of Losses on Selected Tasks}
\label{tab:ab-loss}

\end{table}  

\begin{table}[!h]%\scriptsize
    \centering
    %\vspace{-0.2cm}
    %\caption{Ablation Study of Augmentation Methods}
    %\vspace{-0.3cm}
    
    %%%%%% ============= Source Free ==================
    \resizebox{0.95\linewidth}{18mm}{%
    \begin{tabular}{Sc Sc Sc Sc}
         \toprule
          \textcolor{red}{Single-Source} & A$\rightarrow$ D & Ar$\rightarrow$ Cl & Syn$\rightarrow$ Real \\
          
         \hline
         SD-SHOT & 89.2 & 50.1 & 74.5 \\
         SD-SHOT w/ MixUp & 89.5 (+0.3) & 50.1 (+0.0) & 73.6 (-0.9) \\
         SD-SHOT w/ CutMix & 89.0 (-0.2) & 49.8 (-0.3) & 74.2 (-0.3) \\
         SD-SHOT w/ RandAugment & 89.4 (+0.2) & 50.3 (+0.2) & 74.7 (+0.2)\\
         SD-SHOT w/ \textbf{Phase MixUp} & 90.8 \textcolor{amethyst}{(+1.6)} & 50.7 \textcolor{amethyst}{(+0.6)} & 75.2 \textcolor{amethyst}{(+0.7)} \\
         \hline
         
         DINE-full  w/o MixUp & 90.9 & 52.9 & 76.3 \\
         DINE-full  (w/ MixUp by default)& 91.7 (+0.8) & 54.2 (+1.3) & 77.3 (+1.0)\\
         DINE-full w/ CutMix & 91.4 (+0.5) & 53.8 (+0.8) & 77.1 (+0.8)\\
         DINE-full w/ RandAugment & 92.0 (+1.1) & 54.4 (+1.5) & 77.3 (+1.0)\\
         DINE-full w/ \textbf{Phase MixUp} & 92.8 \textcolor{amethyst}{(+1.9)} & 55.0 \textcolor{amethyst}{(+2.1)} & 77.6 \textcolor{amethyst}{(+1.3)}\\
                
         \end{tabular}} \par\vskip-1.4pt
\resizebox{0.95\linewidth}{18mm}{%
\begin{tabular}{Sc Sc Sc Sc}
\toprule
          \textcolor{blue}{Multi-Source} & $\rightarrow$A & $\rightarrow$P & $\rightarrow$Ar \\
          \hline
         SD-DECISION & 66.6 & 74.1 & 62.5\\
         SD-DECISION  w/ MixUp& 66.5 (-0.1) & 74.3 (+0.2) & 62.0 (-0.5)\\
         SD-DECISION w/ CutMix & 67.2 (+0.6) & 74.7 (+0.6) & 62.2 (-0.3)\\
         SD-DECISION w/ RandAugment & 68.6 (+2.0) & 74.8 (+0.7) & 63.0 (+0.5)\\
         SD-DECISION w/ \textbf{Phase MixUp} & 69.9 \textcolor{amethyst}{(+3.3)} & 75.6 \textcolor{amethyst}{(+1.5)} & 64.5 \textcolor{amethyst}{(+2.0)}\\ \bottomrule
          
         \hline
         DINE-full  w/o MixUp & 75.3 & 78.8 & 73.0\\
         DINE-full  (w/ MixUp by default)& 77.1 (+1.8) & 79.7 (+0.9) & 74.9 (+1.9)\\
         DINE-full w/ CutMix & 75.0 (-0.3) & 77.4 (-1.4) & 73.5 (+0.5)\\
         DINE-full w/ RandAugment & 78.1 (+2.8) & 80.6 (+1.8) & 74.5 (+1.5)\\
         DINE-full w/ \textbf{Phase MixUp} & 78.8 \textcolor{amethyst}{(+3.5)} & 81.3 \textcolor{amethyst}{(+2.5)} & 75.6 \textcolor{amethyst}{(+2.6)}\\ \bottomrule
         \end{tabular}}
\caption{Ablation Study of Augmentation Methods}
\label{tab:ab-aug}
%\vspace{-0.5cm}
\end{table}

\begin{figure*}[htp]
  %\vspace{-5mm}
  \centering
  \includegraphics[width=0.97\linewidth]{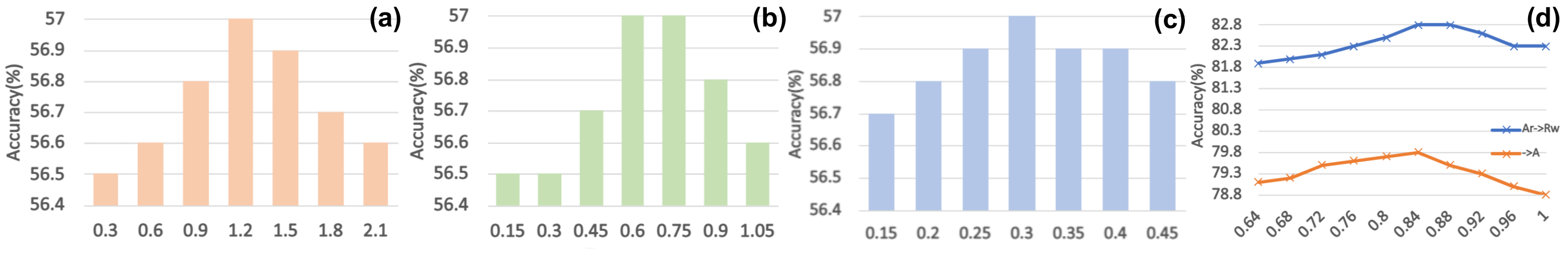}
  %\vspace{-10pt}
  \caption{Parameter Analysis (best viewed in color.): (a) Analysis on $\beta$ in Eq. \ref{eq:total-bbda}; (b) Analysis on $\gamma$ in Eq. \ref{eq:total-bbda}; (c) Analysis on $\theta$ in Eq. \ref{eq:total-bbda}; (d) Analysis on subnetwork width ratios. (a), (b), and (c) are conducted on the task Ar$\rightarrow$ Cl, while (d) is based on Ar$\rightarrow$ Rw and $\rightarrow$A. }
  \label{fig:params}
  %\vspace{-8pt}
\end{figure*}

\begin{figure*}
  %\vspace{-3mm}
  \centering
  \includegraphics[width=0.95\linewidth]{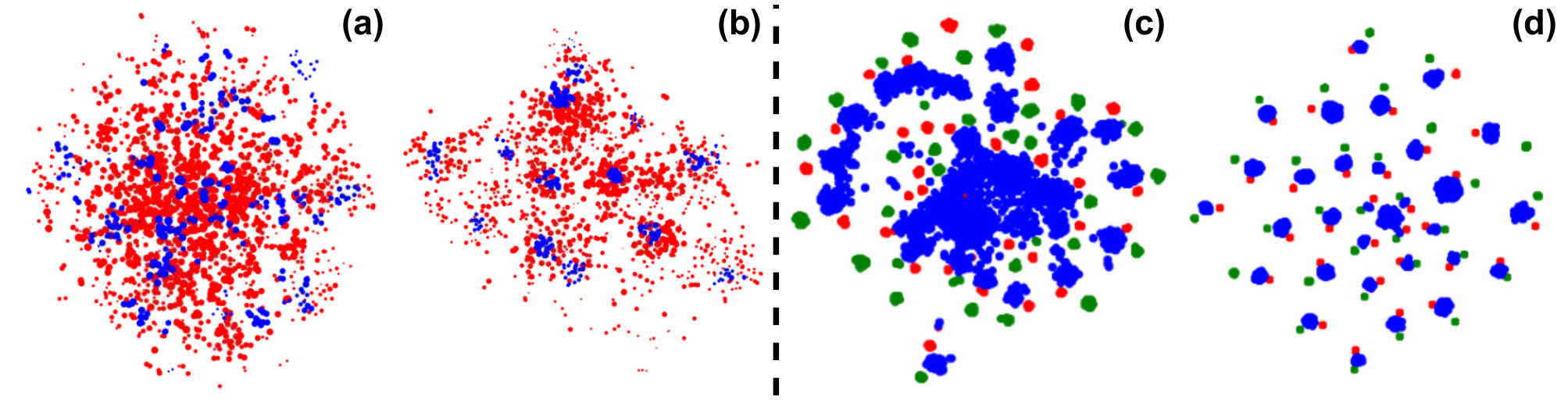}
  %\vspace{-8pt}
  \caption{Feature Visualization. (best viewed in color.): (a) Single-Source before adaptation; (b) Single-Source after adaptation; (c) Multi-Source before adaptation; (d) Multi-Source after adaptation. We select task A (red dots)$\rightarrow$ D (blue dots) for Single-Source and D (red dots) $\&$ W (green dots) $\rightarrow$ A (blue dots) for Multi-Source.}
  \label{fig:tsne}
  %\vspace{-15pt}
\end{figure*}
\paragraph{Ablation Study.} We study the contributions of three proposed losses in our approach and the results are shown in Table \ref{tab:ab-loss}. Three representative tasks from different datasets are selected for both Single-Source and Multi-Source adaptation. Based on the three losses and their dependency (e.g., $\mathcal{L}_{wg}$ relies on the value of $\mathcal{L}_{sd}$), six situations are listed. From Table \ref{tab:ab-loss}, we observe that all the components play a role in improving the model's performance. Besides, it demonstrates that simply using $\mathcal{L}_{sd}$ cannot ensure the efficacy of the distillation from the subnetwork to full network, as the increases are all less than $1\%$. However, when $\mathcal{L}_{sd}$ works together with $\mathcal{L}_{wg}$, the boost is significant enough as no less than $2.3\%$ in all tasks here. This indicates that maintaining the gradient discrepancy of the full target network and target subnetwork is really important, as it ensures that the target subnetwork learns different knowledge, thus guaranteeing the distillation process in $\mathcal{L}_{sd}$ to be meaningful. %Moreover, the analysis that relies on ViT-based source models is presented in the \textbf{Appendix}.% Table \ref{tab:ab-sfda} reports the results under Source-Free setting. For Black-Box we put the results in the \textit{supplementary}. From this table we can see that any removal of these two parts will lead to worse performance. Frequency MixUp plays a more significant role than Mutual Learning. Besides, from the task A$\rightarrow$D we observe that the cooperation of these two components lead to a better result (+1.5$\%$) than the sum of improvements (+1.2$\%$), \textit{indicating that they are not independent modules and can work synergistically.}
As the proposed Phase MixUp can be considered as a general data augmentation scheme, we further validate its superiority by comparing it with the existing state-of-the-art augmentation methods including MixUp \cite{zhang2018mixup}, RandAugment \cite{shorten2019survey} and CutMix \cite{yun2019cutmix}. All these techniques are combined with different Black-Box adaptation approaches. The comparison results of Single-Source and Multi-Source adaptation are listed in Table \ref{tab:ab-aug}. It is evident that Phase MixUp consistently outperforms all the other three techniques, and it plays a positive role in all situations. %Besides, we observe that augmentation methods conducted on both input- and label-level like MixUp and CutMix cannot provide enough assistance to the regularization, and sometimes may degrade the model's performance. This proves that Phase MixUp, which is conducted only on the input level, is more effective. %We show the analysis of ViT-based source models in the \textbf{Appendix}.%Besides, we observe that CutMix harms the results in most situations, showing that not all the augmentation techniques are beneficial to Black-Box adaptation.

%\vspace{0.1cm}
\paragraph{Parameter Study.} There are three hyperparameters in our overall objective (Eqs. \ref{eq:total-bbda}) as $\beta$, $\gamma$, and $\theta$ that weight the importance of $\mathcal{L}_{pm}$, $\mathcal{L}_{sr}$, and $\mathcal{L}_{wg}$. We select the task Ar$\rightarrow$Cl from Office-Home and conduct parameter analysis 
in Fig. \ref{fig:params}a, Fig. \ref{fig:params}b, and Fig. \ref{fig:params}c. For $\beta$, we choose relatively large values with an interval of 0.3, while for $\gamma$ and $\theta$ the values are smaller with an interval of 0.15 and 0.05 respectively. From this figure, we can observe that the best values for $\beta$ are near 1.2 and 1.5. For $\gamma$, 0.60 to 0.75 is a suitable range, and 0.25 to 0.30 is ideal for $\theta$. What's more, we can see that the model's performance remains stable and competitive in the range of values we tested. %The parameter study on the Multi-Source setting is provided in the \textbf{Appendix}. 
We also provide a detailed analysis of the subnetwork width in Subnetwork Distillation based on both Single-Source and Multi-Source settings with the tasks Ar$\rightarrow$Rw and  $\rightarrow$A separately, as shown in Fig. \ref{fig:params}d. We see that the best choice for the Single-Source task is between 0.84 and 0.88, while it is between 0.80 and 0.84 for the Multi-Source task. Our choice of using 0.84 is reasonable. %Besides, we observe that subnetwork width = 1.0 cannot achieve the best results. This is because if the subnetwork also has the width=1.0, the structure regularization aspect of the Subnet Regularizer is completely lost. %In such a case, it is necessary to maintain enough structural differences between the full target network and the target subnetwork. %We also provide the analysis of ViT-based source models in the \textbf{Appendix}.

%\vspace{0.1cm}
\paragraph{Qualitative Visualization.} We use t-SNE \cite{van2008visualizing} to visualize the features produced by source-pretrained models (i.e. No-Adapt) and our models (i.e. Fully-Adapted), and the results are shown in Fig. \ref{fig:tsne}, where the left part is about Single-Source and the right part is about Multi-Source. Fig. \ref{fig:tsne}a and Fig. \ref{fig:tsne}c are results before adaptation, while Fig. \ref{fig:tsne}b and Fig. \ref{fig:tsne}d are results after adaptation. All of these features are produced by Resnet-based source models and target models. From them, we can observe that after adaptation, the data points with varied colors (i.e., from different domains) form multiple clear clusters and are no longer in chaos, which demonstrates the effectiveness of our proposed method \textbf{RAIN}.

%\vspace{-0.15cm}
\section{Conclusion}

In this paper, we propose a method named RegulArization on Input and Network (RAIN) for Black-Box domain adaptation. We propose a new data augmentation technique called Phase MixUp to regularize the data input, thus encouraging class consistency for better target representations. We also propose Subnetwork Distillation as a network-level regularization technique to transfer knowledge from the target subnetwork to the full target network, hence learning diverse target representations and calibrating the full network. Comprehensive experiments on several datasets testify RAIN's efficacy in both Single-Source and Multi-Source settings, together with detailed quantitative and qualitative analysis. 

%\small
%\newpage
\appendix
\section*{Appendix}
In this supplementary material, we provide the following contents for better understanding of the paper:
\begin{enumerate}

    \item Proof of Lemma \ref{lemma} and Theorem \ref{theorem} in the manuscript.
    
    \item Extensive Study on Phase MixUp.
    
    \item Extensive Ablation Study.
    
    %\item Analysis of Frequency MixUp Ratio.
    \item Extensive Parameter Study.
\end{enumerate}

\section{Theoretical Proofs}

In this section, we prove Lemma \ref{lemma} and Theorem \ref{theorem}. 

\begin{lemma}
\label{lemma}
    The source loss and target loss are bounded by joint distributions and empirical predictions:
   
    \begin{equation}
       l_{s} \leq \mathbb{E}_{p_{s}(y,z)}[-\log\hat{p}(y|z)], ~~~
         l_{t} \leq \mathbb{E}_{p_{t}(y,z)}[-\log\hat{p}(y|z)].
     \end{equation}
     
\end{lemma}

\begin{proof}
    \begin{align*}
        l_{s} &= \mathbb{E}_{p_{s}(x,y)}[-\log \hat{p}(y|x)]\\
        &= \mathbb{E}_{p_{s}(x,y)}[-\log \mathbb{E}_{p(z|x)}( \hat{p}(y|z))]\\
        &\leq \mathbb{E}_{p_{s}(x,y)}[\mathbb{E}_{p(z|x)}(-\log \hat{p}(y|z))] &&\text{(Jensen Inequality)}\\
        &= \mathbb{E}_{p_{s}(y,z)}[-\log\hat{p}(y|z)].
    \end{align*}
Similarly, we can get 
$l_{t} \leq \mathbb{E}_{p_{t}(y,z)}[-\log\hat{p}(y|z)]$.
     
\end{proof}

\begin{theorem}
\label{theorem}
The target loss can be bounded by source loss and the JS divergence between the outputs of the full target network and target subnetwork as:
\begin{equation}
    l_{t} \leq l_{s} + C\sqrt{2\mathrm{D_{JS}}(p_{t}(y,z)||p_{sub}(y,z))}.
\end{equation}
\end{theorem}

\begin{proof}
We assume $p_m(y,z)=\frac{1}{2}[p_{t}(y,z)+p_{sub}(y,z)]$.
\begin{align*}
 &l_{t} \leq \mathbb{E}_{p_{t}(y,z)}[-\log\hat{p}(y|z)] \\&\text{(Lemma \ref{lemma})}\\&=  -\int [p_{t}(y,z)-p_{s}(y,z)]\log\hat{p}(y|z)  \\ &- \int p_{s}(y,z)\log\hat{p}(y|z)\\
&\leq -\int [p_{t}(y,z)-p_{sub}(y,z)]\log\hat{p}(y|z)  \\ &- \int p_{s}(y,z)\log\hat{p}(y|z) \\&\text{(Assumption 2)}\\
&= l_{s} -\int [p_{t}(y,z)-p_{sub}(y,z)]\log\hat{p}(y|z) \\&\text{(Lemma \ref{lemma})}\\
&\leq l_{s} -\int_{p_{t}>p_{sub}} [p_{t}(y,z)-p_{sub}(y,z)]\log\hat{p}(y|z)\\
&\leq l_{s} + C\int_{p_{t}>p_{sub}} [p_{t}(y,z)-p_{sub}(y,z)] \\&\text{(Assumption 1)}\\
&\leq l_{s} + \frac{C}{2}\int |p_{t}(y,z)-p_{sub}(y,z)| (*)\\
&= l_{s} + \frac{C}{2}\int |p_{t}(y,z)-p_m(y,z)+p_m(y,z)-p_{sub}(y,z)|\\
&\leq l_{s} + \frac{C}{2}\left[\int |p_{t}(y,z)-p_m(y,z)|+\int|p_{sub}(y,z)-p_m(y,z)|\right]\\
&\leq l_{s} + \frac{C}{2}\left[\sqrt{2\int p_{t}\log \frac{p_{t}}{p_m}}+ \sqrt{2\int p_{sub}\log \frac{p_{sub}}{p_m}}\right]\\&\text{(Pinsker Inequality)}\\
&\leq l_{s} + \frac{C}{2} \cdot \sqrt{2} \cdot \sqrt{2}\sqrt{\int p_{t}\log \frac{p_{t}}{p_m} + \int p_{sub}\log \frac{p_{sub}}{p_m}}\\
&= l_{s} + C\sqrt{\mathrm{D_{KL}}(p_{t}(y,z)||p_m(y,z))+\mathrm{D_{KL}}(p_{sub}(y,z)||p_m(y,z))}\\
&= l_{s} + C\sqrt{2\mathrm{D_{JS}}(p_{t}(y,z)||p_{sub}(y,z))}.
\end{align*}
How do we deduce the step with (*)?

\begin{align*}
    &\int [p_{t}(y,z) - p_{sub}(y,z)]dydz = 0 \\
    \implies & \int_{p_{t}>p_{sub}} [p_{t}(y,z) - p_{sub}(y,z)] \\ &+ \int_{p_{t}<p_{sub}} [p_{t}(y,z) - p_{sub}(y,z)] = 0 \\
    \implies & \int_{p_{t}>p_{sub}} |p_{t}(y,z) - p_{sub}(y,z)| \\ =& \int_{p_{t}<p_{sub}} |p_{t}(y,z) - p_{sub}(y,z)| \\ =&\frac{1}{2} 
     \int |p_{t}(y,z) - p_{sub}(y,z)| 
\end{align*}
In such cases, we proved that minimizing JS divergence benefits the domain adaptation theoretically.
\end{proof}

\begin{figure*}[h]
%\vspace{-10pt}
     \centering
     \begin{subfigure}[b]{0.49\textwidth}
         \centering
         \includegraphics[width=\textwidth]{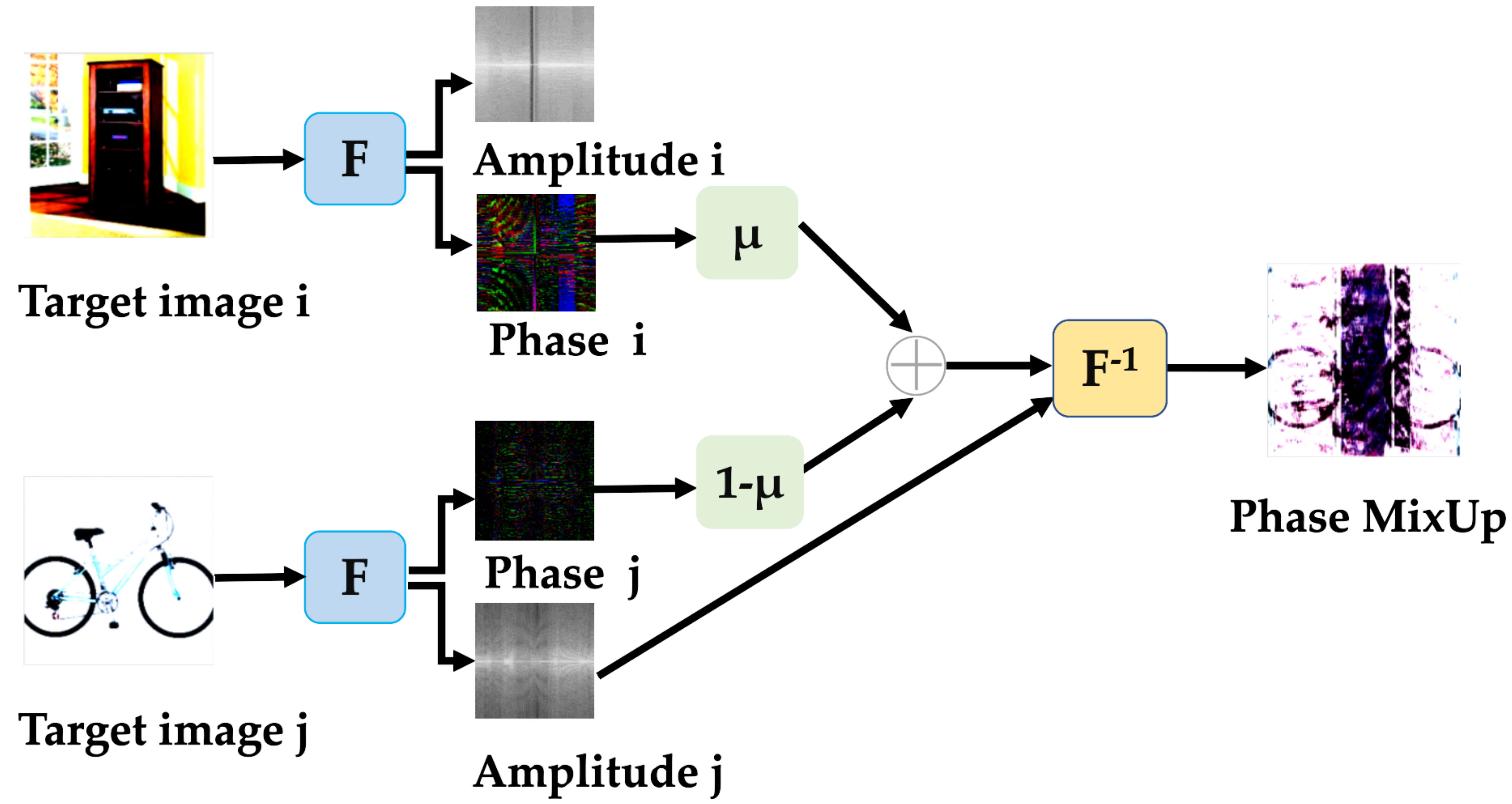}%\vspace{-2mm}
         \caption{Phase MixUp \textbf{(our method)}}
         \label{fig:pha}
     \end{subfigure}
     \hfill
     \begin{subfigure}[b]{0.49\textwidth}
         \centering
         \includegraphics[width=\textwidth]{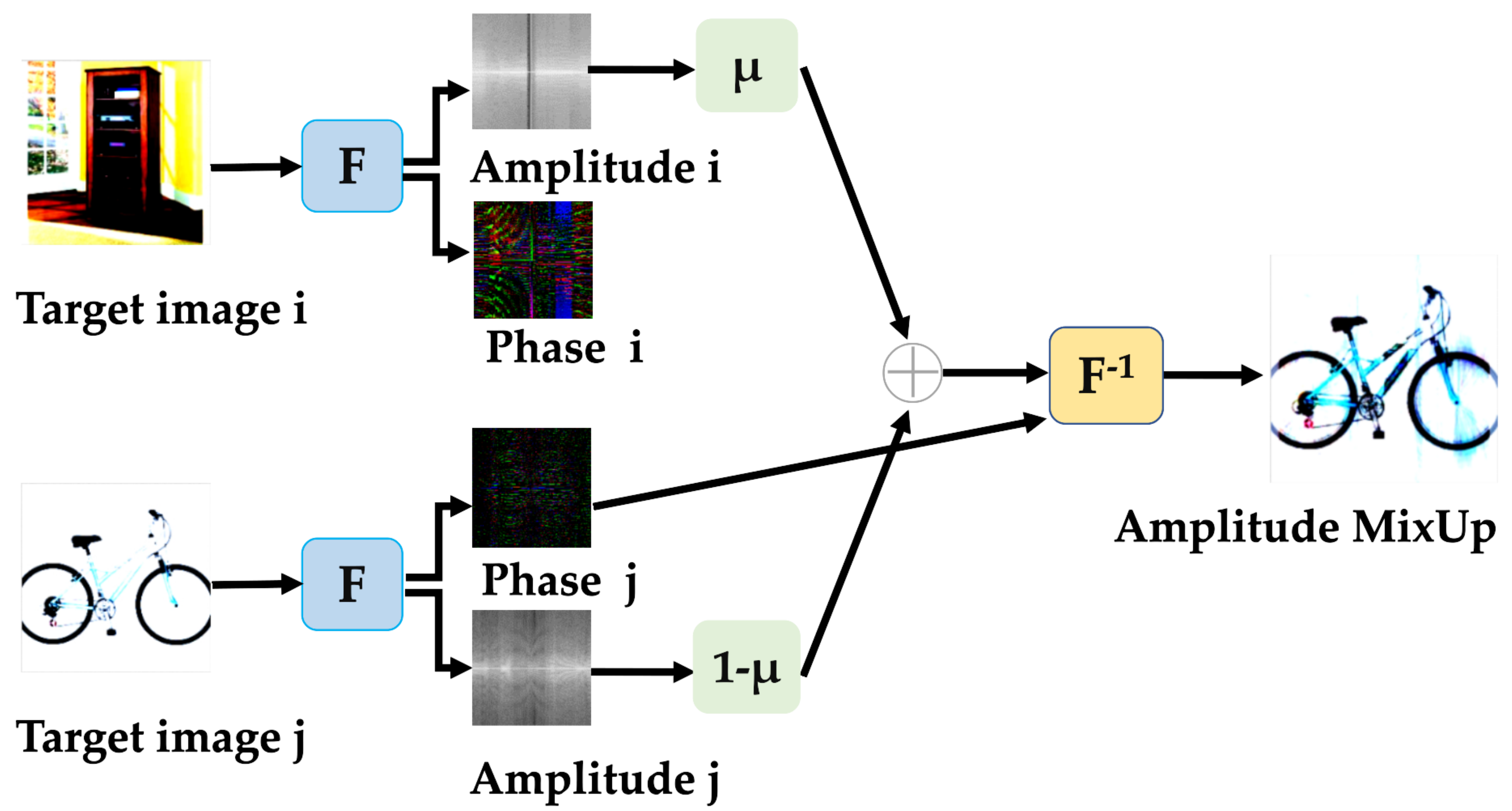}%\vspace{-2mm}
         \caption{Amplitude MixUp}
         \label{fig:amp}
     \end{subfigure}%\vspace{-1mm}
     \hfill
     \begin{subfigure}[b]{0.99\textwidth}
         \centering
         \includegraphics[width=\textwidth]{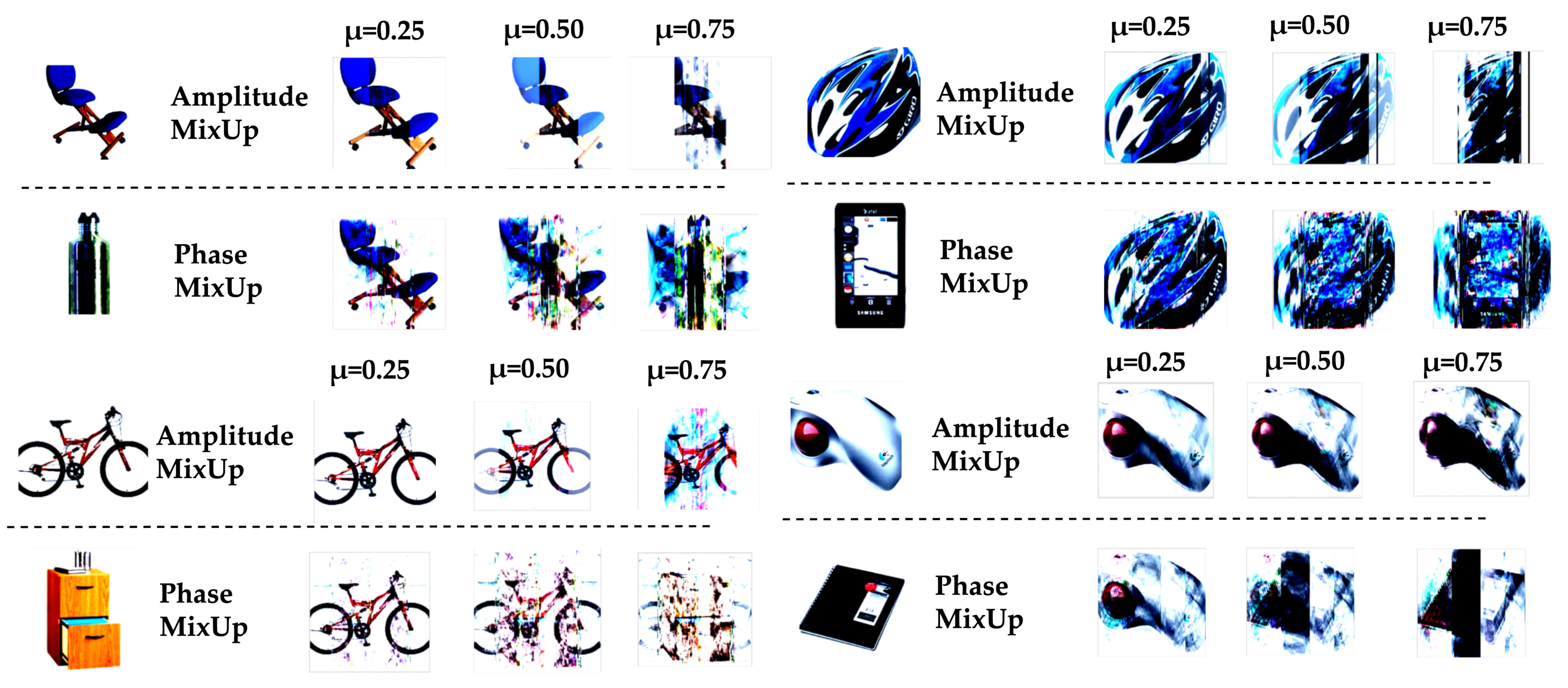}%\vspace{-2mm}
         \caption{}
         \label{fig:compare}
     \end{subfigure}%\vspace{-1mm}
     
\caption{Visualization of Amplitude MixUp and Phase MixUp (Best viewed in color.) $\mu$ is the interpolation ratio. } 
\label{fig:visual_mixup}
%\vspace{-5mm}
\end{figure*}
\section{Extensive Study on Phase MixUp}

Frequency domain analysis (e.g. Fourier transform) is an important image processing tool of images. The frequency spectrum can be decomposed as amplitude spectrum and phase spectrum. Based on the findings of FDA \cite{yang2020fda} and FedDG \cite{liu2021feddg}, the amplitude spectrum includes more information about image style and background, while the phase spectrum emphasizes more object information. 
%However, they didn't provide evidence to illustrate it. And that's why we make such a visualization to explain it as Figure \ref{fig:visual_mixup} shows. 
This inspires us to focus on the \emph{phase spectrum} of images in the proposed Phase MixUp strategy since the object information is the key for the recognition task. To better demonstrate this point, we provide several visual examples to compare Phase MixUp (\textbf{our method}) with Amplitude MixUp in Fig. \ref{fig:visual_mixup}. Here Fig. \ref{fig:pha} shows the process of Phase MixUp, while Fig. \ref{fig:amp} represents the amplitude MixUp process. In Fig. \ref{fig:compare}, we compare the visualizations after the two kinds of mixup processes, and these visualizations can verify the assumption regarding the functions of amplitude and phase. It is evident that Phase MixUp is able to better preserve the object information (e.g. shape and structure) from both input images while reducing the background interference at the same time.
%this opinion is reliable. In such a case, our choice of using phase spectrum to help object recognition tasks is reasonable.

\section{Extensive Ablation Study}

\vspace{-0.3cm}
\begin{table}[!h]%\scriptsize
\caption{Ablation Study of Losses on Selected Tasks}

\vspace{-0.3cm}
\centering
\setlength\tabcolsep{2.0pt} 
\resizebox{0.95\linewidth}{12mm}{%
    \begin{tabular}{lccc}
    \toprule
    \textcolor{red}{Single-Source} & A$\rightarrow$ D & Ar$\rightarrow$ Rw & Syn$\rightarrow$ Real \\
    \hline 
    w/o $\mathcal{L}_{pm}$ \& $\mathcal{L}_{sd}$ \& $\mathcal{L}_{wg}$ (baseline) & 91.7 & 87.5 & 81.5\\
    \hline
    w/ $\mathcal{L}_{pm}$ & 93.0 ($\uparrow$1.3) & 88.8 ($\uparrow$1.3) & 82.9 ($\uparrow$1.4)\\
    w/ $\mathcal{L}_{sd}$ & 92.4 ($\uparrow$0.7) & 88.1 ($\uparrow$0.6) & 82.3 ($\uparrow$0.8)\\
    w/ $\mathcal{L}_{pm}$ \& $\mathcal{L}_{sd}$ & 93.8 ($\uparrow$2.1) & 89.2 ($\uparrow$1.7) & 83.0 ($\uparrow$1.5)\\
    w/ $\mathcal{L}_{sd}$ \& $\mathcal{L}_{wg}$ & 95.0 ($\uparrow$2.3) & 89.6 ($\uparrow$2.1) & 84.8 ($\uparrow$3.3)\\
    RAIN ($\mathcal{L}_{pm}$ + $\mathcal{L}_{sd}$ + $\mathcal{L}_{wg}$) & 96.2 ($\uparrow$4.5) & 90.1 ($\uparrow$2.6) & 86.6 ($\uparrow$5.1)\\
    \bottomrule
   \end{tabular}} \par\vskip-1.4pt
   \setlength\tabcolsep{2.0pt} 
\resizebox{0.95\linewidth}{12mm}{%
    \begin{tabular}{lccc}
    \toprule
    \textcolor{blue}{Multi-Source} & $\rightarrow$A & $\rightarrow$P & $\rightarrow$Ar \\
    \hline 
    w/o $\mathcal{L}_{pm}$ \& $\mathcal{L}_{sd}$ \& $\mathcal{L}_{wg}$ (baseline) & 80.8 & 79.1 & 81.7\\
    \hline
    w/ $\mathcal{L}_{pm}$ & 82.0 ($\uparrow$1.1) & 80.9 ($\uparrow$1.8) & 82.5 ($\uparrow$0.8)\\
    w/ $\mathcal{L}_{sd}$ & 81.5 ($\uparrow$0.7) & 80.5 ($\uparrow$1.4) & 82.2 ($\uparrow$0.5)\\
    w/ $\mathcal{L}_{pm}$ \& $\mathcal{L}_{sd}$ & 82.6 ($\uparrow$1.8) & 81.4 ($\uparrow$2.3) & 83.0 ($\uparrow$1.3)\\
    w/ $\mathcal{L}_{sd}$ \& $\mathcal{L}_{wg}$ & 83.3 ($\uparrow$2.5) & 82.5 ($\uparrow$3.4) & 83.1 ($\uparrow$1.4)\\
    RAIN ($\mathcal{L}_{pm}$ + $\mathcal{L}_{sd}$ + $\mathcal{L}_{wg}$) & 84.5 ($\uparrow$3.7) & 83.2 ($\uparrow$4.1) & 84.0 ($\uparrow$2.3)\\
    \bottomrule
   \end{tabular}}
\vspace{-0.2cm}
\label{tab:ab-loss}
\end{table}  

\begin{table}[!h]%\scriptsize
    \centering
    \vspace{-0.2cm}
    \caption{Ablation Study of Augmentation Methods}
    \vspace{-0.3cm}
    
    %%%%%% ============= Source Free ==================
    \resizebox{0.95\linewidth}{10mm}{%
    \begin{tabular}{cccc}
         \toprule
          \textcolor{red}{Single-Source} & A$\rightarrow$ D & Ar$\rightarrow$ Cl & Syn$\rightarrow$ Real \\
         \hline
         
         DINE-full  w/o MixUp & 94.9 & 63.6 & 82.0 \\
         DINE-full  (w/ MixUp by default)& 95.5 (+0.6) & 64.4 (+0.8) & 85.0 (+3.0)\\
         DINE-full w/ CutMix & 94.5 (-0.4) & 63.3 (-0.3) & 82.1 (-0.1)\\
         DINE-full w/ RandAugment & 95.8 (+0.9) & 65.0 (+1.4) & 86.3 (+4.3)\\
         DINE-full w/ \textbf{Phase MixUp} & 96.4 (+1.5) & 66.1 (+2.5) & 87.5 (+5.5)\\
                
         \end{tabular}} \par\vskip-1.4pt
\resizebox{0.95\linewidth}{10mm}{%
\begin{tabular}{cccc}
\toprule
          \textcolor{blue}{Multi-Source} & $\rightarrow$A & $\rightarrow$P & $\rightarrow$Ar \\
          
         \hline
         DINE-full  w/o MixUp & 79.2 & 81.9 & 83.3\\
         DINE-full  (w/ MixUp by default)& 81.4 (+2.2) & 81.4 (-0.5) &  83.4 (+0.1)\\
         DINE-full w/ CutMix & 78.0 (-1.2) & 80.2 (-1.7) & 82.7 (-0.6)\\
         DINE-full w/ RandAugment & 81.8 (+2.6) & 82.4 (+0.5) & 83.8 (+0.5)\\
         DINE-full w/ \textbf{Phase MixUp} & 83.0 (+3.8) & 84.5 (+2.6) & 84.2 (+0.9)\\ \bottomrule
         \end{tabular}}
\label{tab:ab-aug}
\vspace{-0.2cm}
\end{table}

In this section, we offer an extensive ablation study. Table \ref{tab:ab-loss} is the ablation study of losses based on ViT source models. Here the source and target models are both Resnet-based ones. From it, we observe that all the components benefit the model's performance. Table \ref{tab:ab-aug} is the ablation study of augmentation methods based on ViT source models. The results demonstrate the superiority of Phase MixUp. 

\section{Extensive Parameter Study}

\begin{table}[!h]%\scriptsize
    \centering
    %\vspace{-0.2cm}
    \caption{Ablation Study of Augmentation Methods}
    \vspace{-0.3cm}
    
    %%%%%% ============= Source Free ==================
    \resizebox{0.95\linewidth}{!}{%
    \begin{tabular}{cccccccc}
         \toprule
          \textcolor{red}{$\beta$} & 0.30 & 0.60 & 0.90 & 1.20 & 1.50 & 1.80 & 2.10 \\
         \hline
         
         Accuracy & 79.5 & 79.7 & 79.8 & 79.8 & 79.7 & 79.6 & 79.5\\
                
         \end{tabular}} \par\vskip-1.4pt
\resizebox{0.95\linewidth}{!}{%
\begin{tabular}{cccccccc}
\toprule
          \textcolor{blue}{$\gamma$} & 0.15 & 0.30 & 0.45 & 0.60 & 0.75 & 0.90 & 1.05 \\
         \hline
         
         Accuracy & 79.4 & 79.6 & 79.7 & 79.8 & 79.6 & 79.6 & 76.5\\
         
         \end{tabular}}
\resizebox{0.95\linewidth}{!}{%
\begin{tabular}{cccccccc}
\toprule
          \textcolor{green}{$\theta$} & 0.15 & 0.20 & 0.25 & 0.30 & 0.35 & 0.40 & 0.45 \\
         \hline
         
         Accuracy &  79.3 & 79.5 & 79.6 & 79.8 & 79.8 & 79.7 & 79.6\\ \bottomrule
         
         \end{tabular}}
\label{tab:param-123}
\vspace{-0.2cm}
\end{table}

\begin{table}[!h]%\scriptsize
    \centering
    \vspace{-0.2cm}
    \caption{Ablation Study of Augmentation Methods}
    \vspace{-0.3cm}
    
    %%%%%% ============= Source Free ==================
    \resizebox{0.95\linewidth}{!}{%
    \begin{tabular}{ccccccccccc}
         \toprule
          \textcolor{red}{Ar$\rightarrow$ Rw} & 0.64 & 0.68 & 0.72 & 0.76 & 0.80 & 0.84 & 0.88 & 0.92 & 0.96 & 1.00 \\
         \hline
         
         Accuracy &  89.4 & 89.6 & 89.7 & 89.9 & 90.0 & 90.1 & 90.1 & 90.0 & 89.9 & 89.7\\
                
         \end{tabular}} \par\vskip-1.4pt

\resizebox{0.95\linewidth}{!}{%
\begin{tabular}{ccccccccccc}
\toprule
          \textcolor{blue}{$\rightarrow$A} & 0.64 & 0.68 & 0.72 & 0.76 & 0.80 & 0.84 & 0.88 & 0.92 & 0.96 & 1.00 \\
         \hline
         
         Accuracy &  83.8 & 83.9 & 84.1 & 84.3 & 84.5 & 84.5 & 84.4 & 84.3 & 84.2 & 84.1\\ \bottomrule
         
         \end{tabular}}
\label{tab:param-4}
\vspace{-0.2cm}
\end{table}

In this section, we offer an extensive parameter study. In Table \ref{tab:param-123}, the task $\rightarrow$A in the Office-31 dataset is used to illustrate. From this, we observe the relatively good option for $\beta$, $\gamma$, and $\theta$ are 1.20, 0.60, and 0.30. In Table \ref{tab:param-4}, we offer an analysis of the subnetwork width ratio based on ViT source models, which shows that 0.84 is the best choice. 

\bibliographystyle{named}
\bibliography{ijcai23}

\end{document}